\RequirePackage{fix-cm}
\documentclass[twocolumn]{svjour3}         \usepackage{natbib}
\smartqed  \usepackage{xspace}
\usepackage{times}
\usepackage{amsmath,amssymb,amsfonts}
\usepackage{verbatim}
\usepackage{graphicx}
\usepackage{subfigure}
\usepackage{multirow}
\usepackage{psfrag,graphicx,epsfig,epsf}
\usepackage{graphics}
\usepackage{latexsym}
\usepackage[ruled,linesnumbered, noend]{algorithm2e}
\usepackage{fancyhdr}
\usepackage{fullpage}
\usepackage{wrapfig}
\usepackage{subfigure}
\usepackage{setspace}
\usepackage{color}
\usepackage{url}
\DeclareMathOperator*{\argmin}{arg\!\,min}
\journalname{AURO}
\begin{document}\sloppy

\newcommand{\mam}{$\mathcal{G}_{\textrm{\tt MAM}$}}
\newcommand{\pr}{\ensuremath{\mathbb{P}}}

\newcommand{\reals}{\mathbb{R}}
\newcommand{\integers}{\mathbb{Z}}
\newcommand{\nul}{\textrm{NULL}}

\newcommand{\Wspace}{\mathbb{W}}
\newcommand{\Sspace}{\mathscr{S}}
\newcommand{\Robots}{R}
\newcommand{\Manip}{\mathbb{A}}
\newcommand{\Obstacles}{\mathscr{Z}}
\newcommand{\obst}{\mathbb{Z}}
\newcommand{\cspace}{\ensuremath{\mathbb{C}}}
\newcommand{\ccross}{\ensuremath{\mathbb{C}}}
\newcommand{\ccrossfree}{\ensuremath{\mathbb{C}^{\textup{f}}}}
\newcommand{\cfree}{\ensuremath{\cspace^{\textup{f}}}}
\newcommand{\cinv}{\cspace^{\textup{o}}}
\newcommand{\cbound}{\cspace_{\textrm{\cap}}}
\newcommand{\cstable}{\cspace_{\textrm{s}}}
\newcommand{\cgrasp}{\cspace_{\textrm{G}}}

\newcommand{\oracle}{\mathbb{O}_d}

\newcommand{\qrand}{\ensuremath{Q^{\textup{rand}}}}
\newcommand{\vnear}{\ensuremath{V^{\textup{near}}}}
\newcommand{\vnew}{\ensuremath{V^{\textup{new}}}}
\newcommand{\vlast}{\ensuremath{V^{\textup{last}}}}
\newcommand{\vparent}{\ensuremath{V^{\textup{best}}}}

\newcommand{\Pspace}{\mathbb{P}}
\newcommand{\pose}{p}

\newcommand{\Qspace}{\mathbb{Q}}
\newcommand{\GeomManip}{\mathbb{WM}}

\newcommand{\rad}{\ensuremath{r(n)}}
\newcommand{\radstar}{\ensuremath{r^*(n)}}
\newcommand{\radi}{\ensuremath{r_i(n)}}
\newcommand{\radj}{\ensuremath{r_j(n)}}
\newcommand{\crossrad}{\ensuremath{r_R(n)}}
\newcommand{\crossradstar}{\ensuremath{r^*_R(n)}}
\newcommand{\impcrossrad}{\ensuremath{\hat r_R(n)}}
\newcommand{\allimpcrossrad}{\ensuremath{\hat r_{\textrm{R}(n^R)}}}
\newcommand{\ki}{\ensuremath{k_i(n)}}
\newcommand{\kj}{\ensuremath{k_j(n)}}

\newcommand{\mmgraph}{\ensuremath{\mathbb{G}}}
\newcommand{\mmgimp}{\hat\mmgraph}
\newcommand{\mmgexp}{\mmgraph}
\newcommand{\graph}{\ensuremath{\mathbb{G}}}
\newcommand{\aograph}{\ensuremath{\mathbb{G}^{AO}}}
\newcommand{\tree}{\ensuremath{\mathbb{T}}}
\newcommand{\mmnodes}{\mathbb{\hat V}}
\newcommand{\mmedges}{\mathbb{\hat E}}
\newcommand{\mmnodestpprm}{\mathbb{V}_{\textrm{\chi_i}}}
\newcommand{\mmedgestpprm}{\mathbb{E}_{\textrm{\chi_i}}}
\newcommand{\mmnode}{\mathbb{\hat v}}
\newcommand{\mmedge}{\mathbb{\hat e}}
\newcommand{\nodes}{\mathbb{V}}
\newcommand{\node}{\mathbb{v}}
\newcommand{\edges}{\mathbb{E}}
\newcommand{\edge}{\mathbb{e}}
\newcommand{\prmstar}{\ensuremath{ {\tt PRM^*} }}
\newcommand{\sprmstar}{Soft-\ensuremath{ {\tt PRM} }}
\newcommand{\irs}{\ensuremath{ {\tt IRS} }}
\newcommand{\spars}{{\tt SPARS}}
\newcommand{\drrt}{\ensuremath{{\tt dRRT}}}
\newcommand{\drrtstar}{\ensuremath{{\tt dRRT^*}}}
\newcommand{\udrrtstar}{\ensuremath{{\tt ao\mbox{-} dRRT}}}
\newcommand{\mstar}{\ensuremath{{\tt M^*}}}

\newcommand{\sig}{{\tt SIG}}
\newcommand{\local}{\mathbb{L}}
\newcommand{\rmaps}{\ensuremath{\mathfrak{R}}}

\newcommand{\prm}{{\tt PRM}}
\newcommand{\mmprm}{\ensuremath{\text{Random-}{\tt MMP}}}
\newcommand{\kprmstar}{{\tt k-PRM$^*$}}
\newcommand{\rrt}{\ensuremath{{\tt RRT}}}
\newcommand{\rrtdrain}{{\tt RRT-Drain}}
\newcommand{\rrg}{{\tt RRG}}
\newcommand{\est}{{\tt EST}}
\newcommand{\rrtstar}{\ensuremath{\tt RRT^{\text *}}}
\newcommand{\astar}{{\ensuremath{\tt A^{\text *}}}}
\newcommand{\opens}{P_{\textrm{Heap}}}

\newcommand{\bvp}{{\tt BVP}}
\newcommand{\alg}{{\tt ALG}}
\newcommand{\fixed}{{\tt Fixed}-$\alpha$-\rdg}

\newcommand{\config}{C}

\newcommand{\cost}{\textup{cost}}

\newenvironment{myitem}{\begin{list}{$\bullet$}
{\setlength{\itemsep}{-0pt}
\setlength{\topsep}{0pt}
\setlength{\labelwidth}{0pt}
\setlength{\leftmargin}{10pt}
\setlength{\parsep}{-0pt}
\setlength{\itemsep}{0pt}
\setlength{\partopsep}{0pt}}}{\end{list}}

 \newtheorem{claimthm}{\bf Claim}

\newcommand{\kiril}[1]{{\color{blue} \textbf{Kiril:} #1}}
\newcommand{\chups}[1]{{\color{red} \textbf{Chuples:} #1}}
\newcommand{\rahul}[1]{{\color{green} \textbf{Rahul:} #1}}

\newcommand{\T}{\mathcal{T}}

\newcommand{\dof}{{\tt DoF}}
\newcommand{\dadrrtstar}{\ensuremath{\tt da\mbox{-}dRRT^*}}
\newcommand{\leftrm}{\ensuremath{\mathbb{R}_{\textrm{l}}  }}
\newcommand{\rightrm}{\ensuremath{\mathbb{R}_{\textrm{r}}  }}
\newcommand{\leftmetric}{\ensuremath{\mathbb{P}_{\textrm{l}}  }}
\newcommand{\rightmetric}{\ensuremath{\mathbb{P}_{\textrm{r}}  }}
\newcommand{\cfull}{\ensuremath{\mathbb{C}  }}
\newcommand{\cobs}{\ensuremath{\mathbb{C}_{\textrm{{obs}}}  }}
\newcommand{\cleft}{\ensuremath{\mathbb{C}_{\textrm{{l}}}  }}
\newcommand{\cright}{\ensuremath{\mathbb{C}_{\textrm{{r}}}  }}
\newcommand{\cshared}{\ensuremath{\mathbb{C}_{\textrm{{s}}}  }}
\newcommand{\cgoal}{\ensuremath{q_{\textrm{{goal}}}  }}
\newcommand{\cstart}{\ensuremath{q_{\textrm{{start}}}  }}

\newcommand{\gimpleft}{\ensuremath{\hat\mmgraph_l}}
\newcommand{\gimpright}{\ensuremath{\hat\mmgraph_r}}

\newcommand{\xrand}{\ensuremath{X^{\textup{rand} \ }}}
\newcommand{\xnear}{\ensuremath{X^{\textup{near} \ }}}
\newcommand{\xnew}{\ensuremath{X^{\textup{n}} \ }}
\newcommand{\xlast}{\ensuremath{X^{\textup{last} \ }}}
\newcommand{\xparent}{\ensuremath{X^{\textup{best} \ }}}

\newcommand{\lr}{\ensuremath{\mathbb{R}_{\textrm{ls}}}}
\newcommand{\rr}{\ensuremath{\mathbb{R}_{\textrm{sr}}}}
\newcommand{\lp}{\ensuremath{\mathbb{P}_{\textrm{l}}}}
\newcommand{\rp}{\ensuremath{\mathbb{P}_{\textrm{r}}}}

\newcommand{\motoman}{{\tt Motoman}}
\newcommand{\baxter}{{\tt Baxter}}
\newcommand{\ao}{{\tt AO}}

\newcommand\inlineeqno{\stepcounter{equation}\ (\theequation)}

\newcommand{\ioracle}{\mathbb{I}_d}
\newcommand{\xnewi}{\ensuremath{x^{n}_{\textrm{i} \ }}}

\newcommand{\nit}{\ensuremath{n_{\textrm{it}}}}

\newcommand{\nodenear}{\ensuremath{v^{\textup{near}}}}
\newcommand{\nodenew}{\ensuremath{v^{\textup{new}}}}
\newcommand{\noderand}{\ensuremath{q^{\textup{rand}}}}

\renewcommand{\qed}{\hfill$\square$}
\renewenvironment{proof}{\paragraph{Proof:}}{\qed}

\newcommand{\heuristic}{\ensuremath{\mathbb{H}}}

\definecolor{darkgreen}{RGB}{30,150,30}
\newcommand{\commentdel}[1]{{\color{magenta}}}
 \newcommand{\commentadd}[1]{{#1}} 
\title{\Large \bf dRRT\textsuperscript{*}: Scalable and Informed Asymptotically-Optimal 
Multi-Robot Motion Planning}

\author{Rahul Shome \and Kiril Solovey \and Andrew Dobson \and \\ Dan
  Halperin \and Kostas E. Bekris\thanks{A. Dobson, R. Shome and K. Bekris were
  supported by NSF IIS 1617744 and CCF 1330789.}\thanks{K. Solovey and D. Halperin's work has been supported in 
  part by the Israel Science Foundation (grant no.~825/15) and by the
  Blavatnik Computer Science Research Fund. Kiril Solovey has also been
  supported by the Clore Israel Foundation.}}

\institute{Rahul Shome, and Kostas Bekris \at
              Computer Science Dept. of Rutgers Univ., NJ, USA \\
              \email{\{rahul.shome, kostas.bekris\}@cs.rutgers.com}           \and
          Kiril Solovey, and Dan Halperin \at
              Computer Science Dept. of Tel Aviv Univ., Israel \\
              \email{\{kirilsol, danha\}@post.tau.ac.il }
          \and
          Andrew Dobson \at
              Electrical Engineering and Computer Science Dept. of Univ. of Michigan, MI, USA \\
              \email{chuples@umich.edu}
}

\date{ \vspace{-0.5in} }

\maketitle

\begin{abstract}
Many exciting robotic applications require multiple robots with many degrees of freedom, such as manipulators, to coordinate their motion in a shared
workspace.  Discovering high-quality paths in such scenarios can be
achieved, in principle, by exploring the composite space of all
robots.  Sampling-based planners do so by building a roadmap or a tree
data structure in the corresponding configuration space and can
achieve asymptotic optimality.  The hardness of motion planning,
however, renders the explicit construction of such structures in the
composite space of multiple robots impractical. This work proposes a
scalable solution for such coupled multi-robot problems, which
provides desirable path-quality guarantees and is also computationally
efficient.  In particular, the proposed \drrtstar\ is an informed,
asymptotically-optimal extension of a prior sampling-based multi-robot
motion planner, \drrt. The prior approach introduced the idea of
building roadmaps for each robot and implicitly searching the tensor
product of these structures in the composite space.  This work
identifies the conditions for convergence to optimal paths in
multi-robot problems, which the prior method was not
achieving. Building on this analysis, \drrt\ is first properly adapted
so as to achieve the theoretical guarantees and then further extended
so as to make use of effective heuristics when searching the composite
space of all robots. The case where the various robots share some
degrees of freedom is also studied.  Evaluation in simulation
indicates that the new algorithm, \drrtstar\, converges to
high-quality paths quickly and scales to a higher number of robots
where various alternatives fail. This work also demonstrates the
planner's capability to solve problems involving multiple real-world
robotic arms.

 \end{abstract}

\begin{figure}[h]
	\centering
			\includegraphics[width=0.48\textwidth]{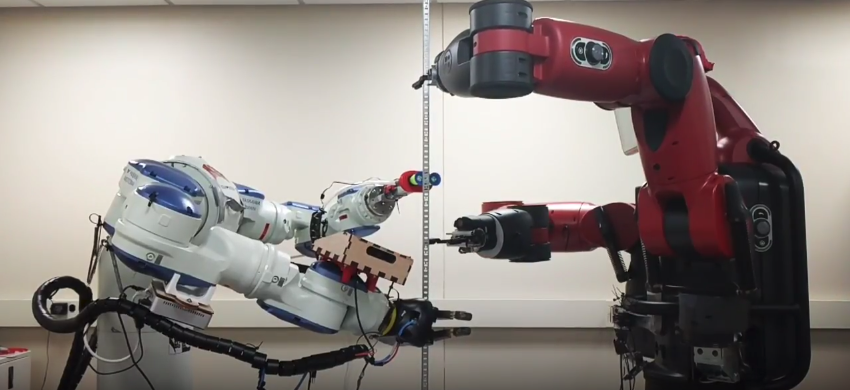}\\
	\vspace{0.1in}
	\includegraphics[width=0.48\textwidth]{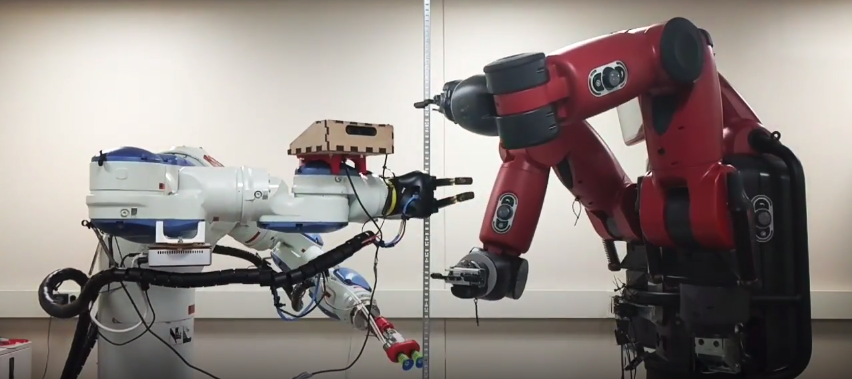}\\
	\vspace{0.1in}
	\caption{Planning in a coupled manner for multiple high-dimensional
		robots is a computationally challenging problem that motivates this
		work. The image shows an instance of a motion planning problem
		solved by the proposed approach. It involves 4 robotic arms, each
		with 7 degrees of freedom, operating in a shared workspace. The arms
		need to move simultaneously from an initial to a goal configuration.
			}
	\label{fig:setup}
	\vspace{0.1in}
\end{figure}

\section{Introduction}

A variety of robotic applications, ranging from manufacturing to
logistics and service robotics, involve multiple robotic \commentadd{systems}
operating in the same workspace. In traditional, industrial domains,
such as car manufacturing, the environment is fully known and
predictable. This allows the robots to operate in a highly scripted
manner by repeating the same predefined motions as fast as
possible. New types of tasks, however, require robotic manipulators
that compute high-quality paths on the fly. For instance, a team of
robotic arms can be tasked to pick and sort a variety of objects that
are dynamically placed on a common surface.  Multiple challenges need
to be addressed in the context of such applications, such as detecting
the configuration of the objects and grasping. This work deals with
the multi-robot motion planning (MMP) problem \citep{Wagner:2015bd,
  Gravot:2003kh, Gharbi:2009fu} in the context of such setups, i.e.,
computing the paths of multiple, high-dimensional systems, such as
robotic arms, that operate in a shared workspace, as shown in
Figure~\ref{fig:setup}.  The focus is to solve MMP in a
computationally efficient way as well as in a coupled manner, which
allows to argue about the quality of the resulting paths.

Planning for multiple, high-dimensional robotic systems is quite
challenging. The motion planning problem is already computationally
hard for a single robot \citep{canny1988complexity} that is a
kinematic chain of rigid bodies.  Thus, most approaches for
multi-robot motion planning either quickly become intractable as the
number of robots increases or alternatively sacrifice completeness and
path quality guarantees.  In particular, problem instances are
especially hard when the robots operate in a shared workspace and in
close proximity. In this case, it is not easy to reason for the robots
in a decoupled manner. Instead, it is necessary to operate in the
composite configuration space of all robots. The space requirements,
however, for solving motion planning instances increase exponentially
with problem dimensionality. The composite space of all robots in MPP
instances is typically very high-dimensional to explore in a
comprehensive and resolution complete manner, such as discretizing it
with a grid and searching over it.

Sampling-based planners aim to help with such dimensionality issues by
approximating the connectivity of the underlying configuration
space. They construct graph-based representations, such as a roadmap
or a tree data structure, which store collision free configurations
and paths through a sampling process. Under certain conditions
regarding the density of the corresponding graph, sampling-based
planners can provide desirable path quality guarantees. Specifically,
they achieve asymptotic optimality, i.e., as the sampling process
progresses, the best path on the graph converges to the optimal one in
the original configuration space. Nevertheless, even sampling-based
planners face significant space challenges in the context of MPP
problem, such as the one shown in Figure~\ref{fig:setup}, which
corresponds to a 28-dimensional space. In particular, it becomes
\commentadd{infeasible} with standard, asymptotically optimal
sampling-based planners to explicitly store a graph in the
corresponding space that will allow the discovery of a solution in
practice. \commentadd{This is due to the large number of samples
  required to cover an exponentially larger volume as the
  dimensionality of the underlying space increases. Asymptotically
  optimal planners must maintain in the order of $logn$ edges per
  sample, where $n$ is the number of samples. Thus, when planning for
  high-dimensional systems, the space requirements of the
  corresponding roadmaps surpass the capabilities of standard
  workstations rather quickly.}

A previously proposed sampling-based planner specifically designed for
multi-robot problems, called \drrt~\citep{SoloveySH16:ijrr}, achieved
progress in this area by leveraging an implicit representation of the
composite space in order to provide both completeness and efficiency.
This implicit representation is a graph, which corresponds to the
tensor product of roadmaps explicitly constructed for each robot. This
allows finding solutions for relatively high-dimensional multi-robot
motion planning problems. Nevertheless, this prior method did not
provide any path quality guarantees. 

One key contribution of this work is to show that the structure of
this implicit \commentadd{representation is guaranteed
  (asymptotically) to contain the optimal path for a set} of robots
moving simultaneously.  Nevertheless, defining an implicit graph that
contains a high-quality solution does not guarantee that the final
solution is optimal unless the search process over this graph is
appropriate.  While a provably optimal search approach, such as
\astar, could be implemented to search this graph, the extremely large
branching factor of the implicit roadmap makes this prohibitively
expensive, especially in the context of anytime planning.  Instead,
this work leverages the observation that a sampling-based method
inspired by {\tt RRT$^*$}, which maintains a spanning tree over the
underlying implicit graph, will return optimal solutions if it allows
rewiring operations during the spanning tree construction.  Namely, it
must converge to the tree with all of the minimum-cost paths starting
from the initial query state to each other node in the graph.
Further, this work shows that for a broad range of cost functions over
paths in this graph can be used while still guaranteeing the proposed
\drrtstar\ approach will asymptotically converge towards such a tree.

This paper is an extension of prior work \citep{Dobson:2017aa}, which
introduced an initial version of the $\drrtstar$ and the sufficient
conditions for generating an asymptotically optimal planner in this
context.  The current manuscript provides the following extensions:

\begin{itemize}
\item A more thorough analysis that shows that the desirable
guarantee can be achieved for an additional distance metric for
multi-robot motion planning;
\item A more detailed description of the method, which has been
  further improved for computational efficiency purposes through the
  appropriate incorporation of heuristics;
\item The method has been extended to handle systems with shared
  degrees of freedom, as shown in related work \citep{shome2017}.
\item The experimental section has been extended to include the new
methods as well as demonstrations on physical platforms.
\end{itemize}

The following section summarizes related prior work on the subject
before Section~\ref{sec:setup} introduces the problem
setup. Section~\ref{sec:methods} describes the underlying structure of
the implicit \textit{tensor-roadmap} and the previous method $ \drrt
$~\citep{SoloveySH16:ijrr}. The changes to $ \drrt $ necessary to
achieve asymptotic optimality and computational efficiency, which
result to the proposed algorithm $ \drrtstar $ are presented in
Section~\ref{sec:aodrrt}. An analysis of the properties of the method
are showcased in Section~\ref{sec:analysis}. The method is extended,
in Section~\ref{sec:shared} to systems with shared degrees of freedom.
Section~\ref{sec:experiments} evaluates the methods experimentally and
demonstrates their performance.

\section{Prior Work}
\label{sec:prior}

The multi-robot motion planning problem (MMP) is notoriously difficult
as it involves many degrees of freedom, and consequently a vast search
space, as each additional robot introduces several additional degrees
of freedom to the problem. Certain instances of the problem can be
solved efficiently, i.e., in polynomial run time, and in a complete
manner, at times even with optimality guarantees on the solution
costs~\citep{tmk-cap13,abhs-unlabeled14,SolYuZamHal15}. However, in
general MMP is computationally
intractable~\citep{hss-cmpmio,sy-snp84,SolHal16j,Johnson-RSS-16}.

Decoupled MMP techniques
\citep{ErdLoz86,GhrOkaLav05,LavHut98b,PenAke02,Berg:2005bh,Berg:2009ve}
reduce search space size by partitioning the problem into several
subproblems, which are solved separately. Then, the different
solutions are combined.  These methods, however, typically lack
completeness and optimality guarantees.  While some hybrid approaches
can take advantage of the inherent decoupling between robots and
provide guarantees \citep{Berg:2009ve}, they are often limited to
discrete domains.  The problem is more complex when the robots exhibit
non-trivial dynamics \citep{Peng2005Kinodynamic-Coord}.  Collision
avoidance or control methods can scale to many robots, but lack path
quality guarantees
\citep{vandenberg2011Reciprocal-Collision,Tang2015Complete-Multi}.

In contrast to that, centralized approaches
\citep{kh-pppi05,OdoLoz89,shh12,SoloveySH16:ijrr,Svestka:1998ud,Wagner:2015bd}
usually work in the combined high-dimensional configuration space, and
thus tend to be slower than decoupled techniques. However, centralized
algorithms often come with stronger theoretical guarantees, such as
completeness. Through the rest of this section we will consider
centralized methods, with an emphasis on sampling-based approaches.

Sampling-based algorithms for a single robot
\citep{Kavraki1996Probabilistic-R,LaValle2001,Karaman2011Sampling-based-}
can be extended to the multi-robot case by considering the fleet of
robots as one composite robot \citep{sl-upp}. Such an approach suffers
from inefficiency as it overlooks aspects of multi-robot planning, and
hence can handle only a very small number of robots. Several
techniques tailored for instances of MMP involving a small number of
robots have been described~\citep{hh-hmp,shh12}.

In previous work~\citep{sh-kcolor14}, an extension of MMP was
introduced, which consists of several groups of interchangeable
robots. At the heart of the algorithm is a novel technique where the
problem is reduced to several discrete pebble-motion problems
\citep{Kornhauser:1984oq,LunBek11,YuLav13ICRA-A}.  These reductions
amplify basic samples into massive collections of free placements and
paths for the robots.  An improved
version~\citep{KroETAL14} of this algorithm applied it to rearrange
multiple objects using a robotic manipulator.

Previous work~\citep{Svestka:1998ud} introduced a different approach,
which leverages the following fundamental observation: the structure
of the overall high-dimensional multi-robot configuration space can be
inferred by first considering independently the free space of every
robot, and combining these subspaces in a meaningful manner to account
for robot-robot collisions.  They suggested an approach which combines
roadmaps constructed for individual robots into one
\emph{tensor-product} roadmap $\mmgimp$, which captures the structure
of the joint configuration space (see more information in
Section~\ref{sec:methods}).

Due to the exponential nature of the resulting roadmap, this technique
is only applicable to problems that involve a modest number of robots.
A recent work~\citep{Wagner:2015bd} suggests that
$\mmgimp$ does not necessarily have to be explicitly represented. They
apply their \mstar\ algorithm to efficiently retrieve paths over
$\mmgimp$, while minimizing the explored portion of the roadmap.  The
resulting technique is able to cope with a large number of robots, for
certain types of scenarios. However, when the degree of simultaneous
coordination between the robots increases, there is a sharp increase in
the running time of this algorithm, as it has to consider many
neighbors of a visited vertex of $\mmgimp$. This makes \mstar less
effective when the motion of multiple robots needs to be tightly
coordinated.

Recently a different sampling-based framework for MMP was introduced,
which combines an implicit representation of $\mmgimp$ with a novel
approach for pathfinding in geometrically-embedded graphs tailored for
MMP~\citep{SoloveySH16:ijrr} .  The \emph{discrete-RRT} (\drrt)
algorithm is an adaptation of the celebrated \rrt\ algorithm for the
discrete case of a graph, and it enables a rapid exploration of the
high-dimensional configuration space by carefully walking through an
implicit representation of the tensor product of roadmaps for the
individual robots (see extensive description in
Section~\ref{sec:methods}).  The approach was demonstrated
experimentally on scenarios that involve as many as $60$ DoFs and on
scenarios that require tight coordination between robots.  On most of
these scenarios \drrt\ was faster by a factor of at least ten when
compared to existing algorithms,including the aforementioned \mstar.

Later, \drrt\ was applied to motion planning of a free-flying
multi-link robot~\citep{SalSolHal16}. In that case, \drrt\ allowed to
efficiently decouple between costly self-collision checks, which were
done offline, and robot-obstacle collision checks, by traversing an
implicitly-defined roadmap, whose structure resembles to that of
$\mmgimp$. \drrt\ has also been used in the study of the effectiveness
of metrics for MMP, which are an essential ingredient in
sampling-based planners~\citep{AtiSolHal17}.

The current work proposes \drrtstar\ and shows that it is an efficient
asymptotically optimal extension of the previously proposed \drrt.
The \drrtstar\ framework is an anytime algorithm, which quickly finds
initial solutions and then refines them, while ensuring asymptotic
convergence to optimal solutions.  Simulations show that the method
practically generates high-quality paths while scaling to complex,
high-dimensional problems, where alternatives fail.

\section{Problem Setup and Notation}\label{sec:setup}

We start with a definition of the problem. Consider a shared workspace
with $R \geq 2$ holonomic robots, each operating in a $ d
$-dimensional configuration space $\cspace_i \subset \mathbb{R}^d$ for
$1\leq i\leq R$.  For a given robot $i$, denote its free space, i.e.,
the set of all collision free configurations, by $\cfree_i
\subset \cspace_i$, and the obstacle space by
$\cinv_i=\cspace_i\setminus \cfree_i$.

The \emph{composite configuration space}
$\cspace = \prod^R_{i=1} \cspace_i$ is the Cartesian product of each
robot's configuration space.  That is, a composite configuration
$Q = (q_1,\ldots,q_R) \in \cspace$ is an $R$-tuple of robot
configurations.  For two distinct robots~$i,j$, denote by
$I_i^j(q_j)\subset \cspace_i$ the set of configurations of robot $i$,
which lead into collision with robot $j$ at its configuration $q_j$.
Then, the composite free space $\cfree \subset \cspace$ consists of
configurations $Q=(q_1,\ldots,q_R)$ in which robots do not collide with
obstacles or pairwise with each other. Formally: 
\begin{myitem}
\item $q_i \in \cfree_i$ for every $1\leq i\leq R$;
\item $q_i \not\in I_i^j(q_j), q_j \not\in I_j^i(q_i)$ for every 
$1 \leq i < j\leq R$.
\end{myitem}
The composite obstacle space is defined as
\mbox{$\cinv = \cspace \setminus \cfree$}.  

Multi-robot planning is concerned with finding (collision-free)
composite trajectories of the form $\Sigma:[0,1]\rightarrow \cfree$.
$\Sigma$ is an $R$-tuple $(\sigma_1, \ldots, \sigma_R)$
of single-robot trajectories 
\mbox{$\sigma_i:[0,1] \rightarrow \cspace_i$}.

This work is concerned with producing high-quality trajectories, which
minimize certain \emph{cost functions}. In particular, we consider
three cost functions $\cost(\cdot)$, which are presented below. Let
$\Sigma=(\sigma_1,\ldots,\sigma_R)$ be a composite trajectory. For the
following, $\|\cdot\|$ denotes the standard \emph{arc length} of a
curve:
\begin{myitem}
\item The sum of path lengths: $\cost(\Sigma)= \sum_{i=1}^R \|\sigma_i\|$.
\item The maximum path length: $\cost(\Sigma)= \max_{i=1:R} \|\sigma_i\|$.
\item The Euclidean arc length of $\Sigma$:
  $\cost(\Sigma) = \|\Sigma\| $
\end{myitem}

\begin{figure}[t]
	\centering
	\includegraphics[width=3in]{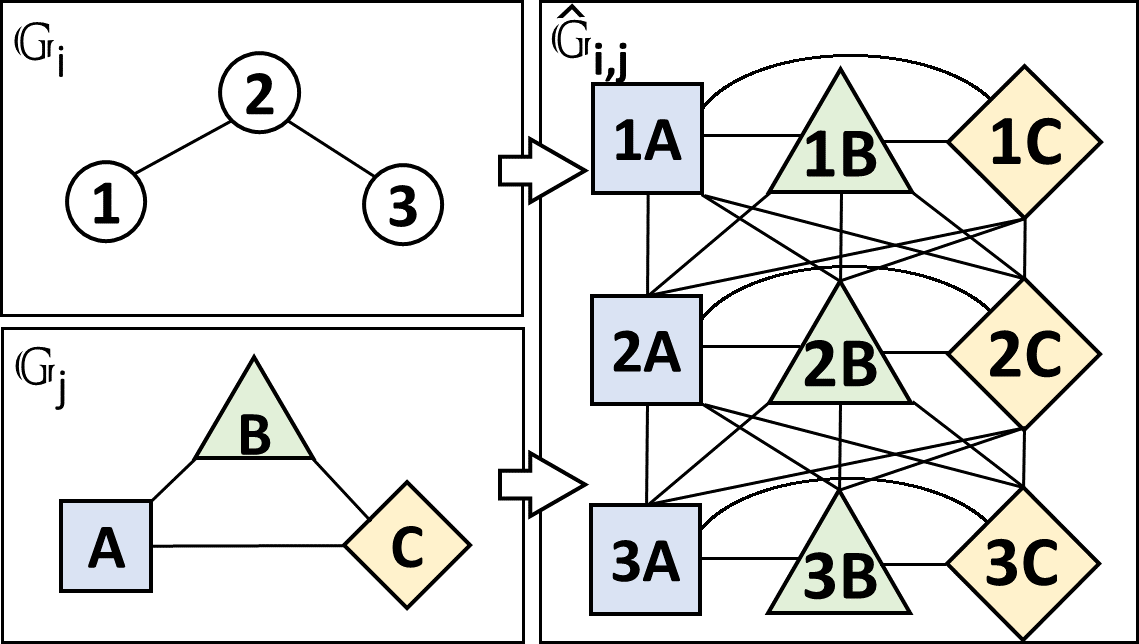}
	\caption{\commentadd{Given the individual robot roadmaps $\graph_i$ and
          $\graph_j$ shown on the left, the tensor product roadmap
          $\hat \graph_{i,j}$ arises, which is shown on the right.
          For each pair of nodes, where one is selected from
          $\graph_i$ and the other from $\graph_j$, a node is defined
          in $\hat\graph_{i,j}$. Two nodes in the tensor-product
          roadmap share an edge if their constituent nodes in the
          individual robot roadmaps also share an edge. For instance,
          nodes $1A$ and $3C$ do not share an edge in $\hat
          \graph_{i,j}$ because nodes $1$ and $3$ do not share an edge
        in $\graph_i$.}}
	\label{fig:tprm}
\end{figure}

This work presents \drrtstar\ as an efficient, anytime solution to the 
robustly-feasible composite motion planning (RFCMP) problem :

\begin{definition}[RFCMP]
Given $R$ robots operating in composite configuration space $\cspace
= \prod^R_{i=1} \cspace_i$, and for a given query $S=(s_1,\ldots,s_R),
T=(t_1,\ldots,t_R)$, an RFCMP 
problem is one which yields a robustly-feasible trajectory 
$\Sigma: [0,1] \rightarrow \cfree$ and $\Sigma(0)=S, \Sigma(1)=T$. Namely, there exists a fixed constant $\delta > 0$ such that 
  \commentadd{$\forall\ \tau\in [0,1],X\in \cinv$} it holds that
  $$\|\Sigma(\tau)-X\|\geq \delta.$$
\end{definition}
One of the primary objectives of this work is to provide asymptotic
optimality in the composite configuration space without explicitly
constructing a planning structure in this space.  

\begin{definition}[Asymptotic Optimality]
Let $ m $ be the time budget of the algorithm and a \textit{robustly optimal} solution $ \Sigma^{(m)} $ of cost $ c^* $ is returned after time $ m $, then asymptotic optimality is defined as ensuring that the following holds true for any $ \epsilon>0 $.
  $$\lim_{m\rightarrow \infty}\Pr\left[\cost(\Sigma^{(m)})\leq
    (1+\epsilon)c^*\right]=1.$$
\end{definition}

\section{Algorithmic Foundations}\label{sec:methods}

This section provides a detailed description of the
\emph{discrete}-\rrt\ (\drrt) method~\citep{SoloveySH16:ijrr}, which
is the basis of our method presented in Section~\ref{sec:aodrrt}.
\drrt\ was posed as an efficient way to search an implicitly defined
tensor-product roadmap, which captures the structure of $\cspace$
without explicitly sampling this space.

\subsection{Tensor-product roadmap}
Here we provide a formal definition of the tensor-product roadmap that
\drrt\ is designed to explore. For every robot $1\leq i\leq R$
construct a \prm\ graph~\citep{Kavraki1996Probabilistic-R}, denoted by
$\graph_i = (\nodes_i, \edges_i)$, which is embedded in
$\cfree_i$. That is, $\graph_i$ can be viewed as an approximation of
$\cfree_i$ and encodes collision free motions for robot $i$. The
construction of $\graph_i$ is determined by two parameters $n$ and
$r_n$, which represent the number of samples, and the connection
radius, respectively. As will be discussed in the following sections,
it is necessary the roadmaps $\graph_1,\ldots,\graph_R$ to be
constructed with certain range of parameters to guarantee asymptotic
optimality of the new planners (Section~\ref{sec:aodrrt}).

Define the \emph{tensor-product roadmap}, denoted by
$\mmgimp = (\mmnodes, \mmedges)$, as the tensor product between
$\graph_1,\ldots,\graph_R$ (see Figure~\ref{fig:tprm}). Each vertex of
$\mmgimp$ describes a simultaneous placement of the $R$ robots, and
similarly an edge of $\mmgimp$ describes a simultaneous motion of the
robots.  Formally,
$\mmnodes = \{ (v_1, v_2, \dots, v_R ): \forall\ i,\ v_i \in \nodes_i
\}$
is the Cartesian product of the nodes from each roadmap~$\graph_i$.
For two vertices
$V =(v_1,\ldots,v_m) \in \mmnodes, V'=(v'_1,\ldots,v'_m) \in
\mmnodes$,
the edge set $\mmedges$ contains edge $(V,V')$ if
$\forall i \in [1,R]:\ v_i=v'_i$ or
$(v_i,v'_i)\in \edges_i$.\footnote{Notice this difference from
  the original $\drrt$~\citep{SoloveySH16:ijrr} so as to allow edges
  where some robots remain motionless.}  Note that
by the definition of $\graph_1,\ldots,\graph_R$, the motion described by
each edge $E\in \mmedges$ represents a path for the $R$ robots in
which the robots do not collide with obstacles. However, collisions
between pairs of robots still may be possible.

It is important to note that the \textit{tensor-product roadmap} has $\| \mmnodes \| = \prod_{i=1}^{R} \|\mmnodes_i\| $ vertices. Given the neighborhood of a node $ v_i $ in $ \graph_i $ as $ \mathtt{Adj}(v_i,\graph_i) $, the size of the neighborhood of a node $ v = \{ v_1 \dots v_R \} $ in $ \mmgimp $ is $ \| \mathtt{Adj}(v,\mmgimp) \| = \prod_{i=1}^{R} \| \mathtt{Adj}(v_i,\graph_i) \| $. 
Using the much smaller $\graph_1,\ldots,\graph_R$ to construct $ \mmgimp $ online is computationally beneficial.

\commentadd{The presented algorithms share a common set of input and
  output parameters, such as the configuration space decompositions,
  which are predefined. In practice, the algorithms use pre-computed
  roadmaps in each constituent space online. The collision volumes
  that correspond to the robot and obstacles in the scene are also
  used online for validation. The algorithms output a trajectory in
  the configuration space of all robots, which is collision free with
  all obstacles and among robots.}

\subsection{Discrete RRT}

An explicit construction of $\mmgimp$ is possible in very limited settings that either involve few robots, e.g., $R=2$, or when the underlying single-robot roadmaps have few vertices and edges. However, in general it is prohibitively costly to fully represent it due to its size, which grows exponentially with the number of robots, in terms of the number of vertices. Moreover, in some cases it may be even a challenge to represent all the edges adjacent to a single vertex of $\mmgimp$, as there may be exponentially many of those.

\begin{figure}[!h]
	\centering
	\includegraphics[height=1.4in]{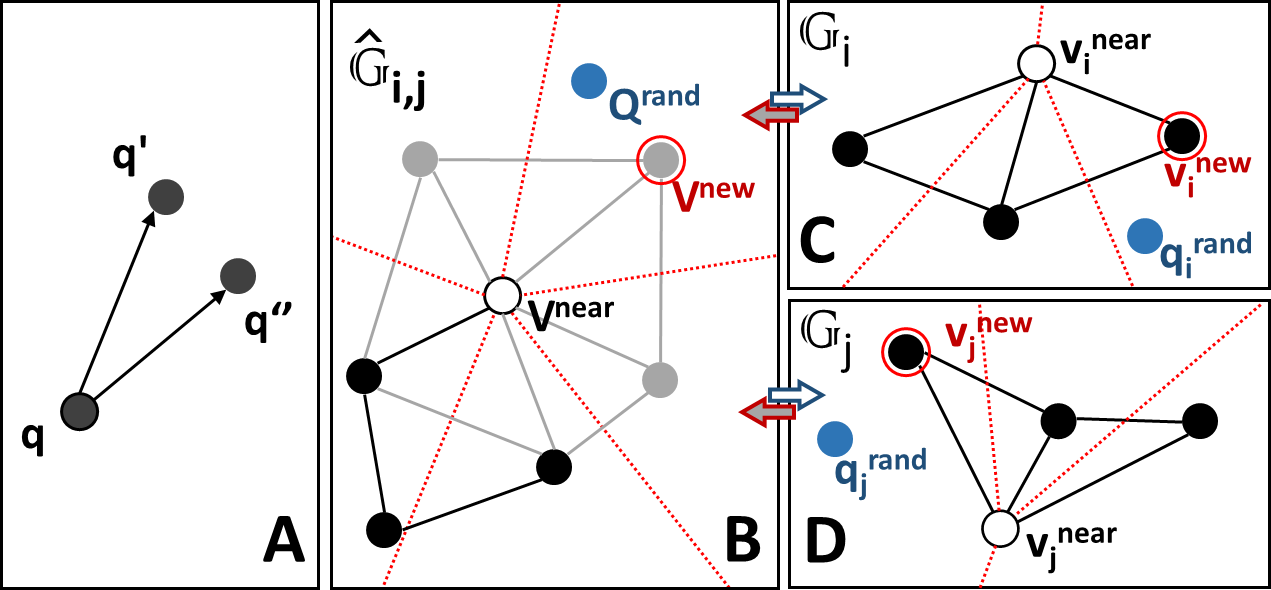}
	\caption{(A) The method reasons over all neighbors $q'$ of $q$ so as
		to minimize the angle $\angle_{q}(q', q'')$. (B)
		$\oracle(\cdot,\cdot)$ finds graph vertex \vnew\ by minimizing angle
		$\angle_{\vnear}(\vnew,\qrand)$. (C,D) \vnear and \qrand are
		projected into each robot's $\cspace$-space to find $\nodenew_{\textup{i}}$ and $\nodenew_{\textup{j}}$,
		respectively, which minimize both $\angle_{
			\nodenear_{\textup{i}}}
		(\nodenew_{\textup{i}}, \noderand_{\textup{i}}
		)$ and $\angle_{
					\nodenear_{\textup{j}}}
				(\nodenew_{\textup{j}}, \noderand_{\textup{j}}
				)$.}
	\label{fig:oracle}
\end{figure}

\begin{figure*}[!ht]
	\centering
	\includegraphics[height=1.49in]{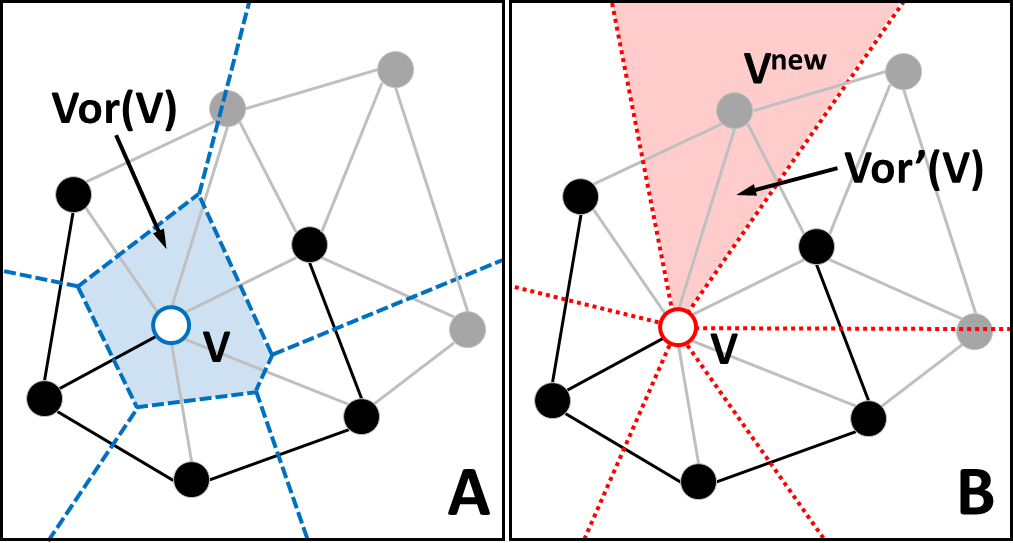}
	\includegraphics[height=1.49in]{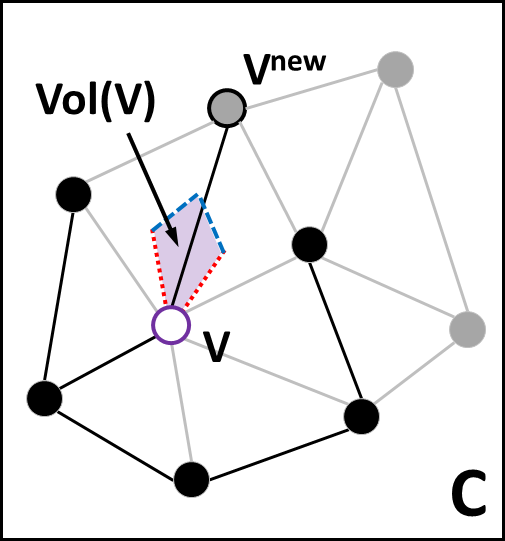}
	\caption{(A) The Voronoi region $\textup{Vor}(V)$ 
		of vertex $V$ is shown where if \qrand\ is drawn, vertex $V$ is 
		selected for expansion. (B) When \qrand\ lies in the directional 
		Voronoi region $\textup{Vor}'(V)$, the expand step expands to \vnew.  
		(C) Thus, when \qrand\ is drawn within volume $\textup{Vol}(V) = 
		\textup{Vor}(V) \cap \textup{Vor}'(V)$, the method will generate 
		\vnew\ via $V$.}
	\label{fig:voronoi}
\end{figure*}

The \drrt\ algorithm enjoys the rich structure that $\mmgimp$ offers (see Section~\ref{sec:analysis})
without explicitly representing it. In particular, it gathers
information on $\mmgimp$ only from the single-robot roadmaps
$\graph_1,\ldots,\graph_R$.

 Similarly to the single-robot planner
\rrt~\citep{LaValle2001}, \drrt\ grows a tree rooted at the start state
of the given query (Line~1). 
\commentadd{\drrt\ restricts the growth of
its tree $\tree$ to the tensor-product roadmap~$\mmgimp$ in contrast to \rrt, which explores the entire space $\cspace$.} That is,
$\tree$ is a subgraph of $\mmgimp$, and $\tree\subset\cfree$. 

The high level operations of the \drrt\ approach are outlined in
Algorithm~\ref{algo:drrt}. The approach will iterate until a solution
is found or the time limit is exceeded (Algorithm~\ref{algo:drrt}, Line~2), beginning with performing a fixed number $\nit$
expansion steps at the beginning of each iteration (Lines~3,~4).  This
expansion process is outlined in Algorithm~\ref{algo:drrt_expand}.
The approach then checks to see if there is a connected path to the
target (Line~5), and once a path is found, it is returned
(Lines~6,~7). 

The expansion procedure begins by drawing a random sample
$\qrand\in\cspace$ (Line~1).  It then finds the nearest neighbor $\vnear$
in the tree (Line~2) and then selects a neighbor $\vnew$, such that
$(\vnear,\vnew)\in \mmedges$, according to a \emph{direction oracle}
function $\oracle$ (Line~3). Then, if $\vnew$ is not in the tree
(Line~4), it is added to the tree (Line~5) and an edge from $\vnear$
to $\vnew$ is also added (Line~6).

We now elaborate on $\oracle$. Given $\vnear,\qrand$, the oracle
returns a vertex $\vnew\in \mmnodes$ that is the neighbor of $\vnear$
(in $\mmgimp$) found in the direction of $\qrand$. The crux of the approach is that $\oracle$ can come up with such a neighbor
efficiently without relying on explicit representation of $\mmgimp$.
Let $Q,Q', Q''\in \cspace$ and define $\rho(Q,Q')$ to be the ray through $Q'$ starting at $Q$.  Then, denote $\angle_{Q} (Q',Q'')$ as the
minimum angle between $\rho(Q,Q')$ and $\rho(Q,Q'')$.  Denote by
$\mathtt{Adj}(\vnear,\mmgimp)$ the set of neighbor nodes of $\vnear$ in $\mmgimp$,
i.e., for every $V\in \mathtt{Adj}(\vnear,\mmgimp)$ it holds that
$(\vnear, V)\in \mmedges$. Then
\commentadd{$$\oracle(\vnear,\qrand)=\argmin_{V\in \mathtt{Adj}(\vnear,\mmgimp)} \angle_{\vnear}(\qrand,V).$$}

\begin{algorithm}[h]
\caption{$\drrt {\tt(} \mmgimp, S, T, \nit {\tt)}$}
\label{algo:drrt}
$\tree.\mathtt{init}(S)$\;
\While{${\tt time.elapsed}() < {\tt time\_limit}$}
{
    \For{$i : 1 \to \nit$}
    {
        ${\tt Expand(} \mmgimp, \tree {\tt)}$\;
    }
    $\pi \gets {\tt Connect\_to\_Target(} \mmgimp, \tree, T {\tt)}$\;
    \If{$\pi \neq \emptyset$}
    {
        {\bf return} ${\tt Trace\_Path(} \tree, T {\tt)}$\;
    }
}
{\bf return} $\emptyset$
\end{algorithm}
 \begin{algorithm}[h]
\caption{${\tt Expand(} \mmgimp, \tree {\tt)}$}
\label{algo:drrt_expand}
$\qrand \gets {\tt Random\_Sample()}$\;
$\vnear \gets {\tt Nearest\_Neighbor(} \tree, \qrand {\tt)}$\;
$\vnew \gets {\tt \oracle (} \vnear, \qrand {\tt)}$\;
\If{ $\vnew \notin \tree$ }
{
    $\tree.{\tt Add\_Vertex(} \vnew {\tt)}$\;
    $\tree.{\tt Add\_Edge(} \vnear, \vnew {\tt)}$\;
}
\end{algorithm}
 
The implementation of $\oracle$ (Algorithm \ref{algo:drrtstar_oracle}) proceeds in the following manner (see
a two-robot case illustrated in Figure \ref{fig:oracle}). Let
$\qrand=(\noderand_1,\ldots,\noderand_R),
\vnear=(\nodenear_1,\ldots,\nodenear_R)$. For every robot $1\leq i\leq R$, the oracle extracts from $\graph_i$ the neighbor $\nodenew_i$ of $\nodenear_i$, which minimizes the expression $\angle_{\nodenear_i}(\noderand_i,\nodenew_i)$. Notice that such a search can be performed efficiently as it only requires to traverse all the neighbors of $\nodenear_i$ in $\graph_i$. 
The combination of all $\nodenear_i$ yields $\vnear$.

As in $\rrt$, $\drrt$ has a Voronoi-bias property
~\citep{DBLP:conf/icra/LindemannL04}.  \commentadd{Showing that $\drrt$
exhibits Voronoi bias is slightly more involved compared to the basic
$\rrt$. This is illustrated in Figure \ref{fig:voronoi}.}  To generate
an edge $(V, \vnew)$, random sample \qrand\ must be drawn within the
Voronoi cell of $V$, denoted as $\textup{Vor}(V)$
(Figure~\ref{fig:voronoi}(A)) and in the general direction of $\vnew$,
denoted as $\textup{Vor}'(V)$ (Figure~\ref{fig:voronoi}(B)).  
The intersection of these two volumes: $\textup{Vol}(V) = \textup{Vor}(V)
\cap \textup{Vor}'(V)$, is the volume to be sampled so as to generate
\vnew via \vnear\ as shown in Figure \ref{fig:voronoi}.

\commentadd{The high-level loop of the algorithm remains similar
  across the method variants.  The input parameter $n_{it}$ denotes
  how many times the tree is expanded before the algorithm checks
  whether a solution has already been discovered.  If $n_{it}=1$, this
  check is performed every iteration. If tracing the path is an
  expensive operation - typically it corresponds to a heuristic search
  process over the tensor product roadmap - then the implementer can
  choose to use a higher value.}

\section{Asymptotically Optimal Discrete RRT}
\label{sec:aodrrt}

This section outlines two versions of the proposed asymptotically optimal
variant of the $\drrt$ method.  The first is a simple uninformed
approach, which relies on the fact that to provide asymptotic
optimality, it is sufficient to use a simple rewiring scheme.  This
simplified version will be called the asymptotically-optimal $\drrt$
($\udrrtstar$).  For the sake of algorithmic efficiency however, a
second, more advanced version is also proposed referred to as
$\drrtstar$.  To summarize the algorithmic contributions of the current work over the original $\drrt$:
\begin{itemize}
\item $\drrtstar$ performs a rewiring step to refine paths in the
tree, reducing costs to reach particular nodes.
\item $\drrtstar$ is anytime, employing branch and bound pruning 
after an initial solution is reached.
\item $\drrtstar$ promotes progress towards the goal during tree node selection.
\item $\drrtstar$ employs an informed expansion procedure $\ioracle$ capable of using heuristic guidance.
\end{itemize}
\subsection{ao-$\drrt$}
This section outlines $\udrrtstar$, an asymptotically optimal version of the $\drrt$
algorithm which has been minimally modified to guarantee asymptotic
optimality.  At a high-level the approach uses a tree re-wiring 
technique reminiscent of $\rrtstar$~\citep{Karaman2011Sampling-based-}.

Algorithm~\ref{algo:udrrtstar} outlines $\udrrtstar$ which iteratively expands a tree $\tree$ over $\mmgimp$ given a time budget 
(Algorithm~\ref{algo:udrrtstar}, Line~2), performing $\nit$ consecutive calls to 
${\tt Expand\_\udrrtstar}$ (Lines~3,~4).  
Then, the method attempts to 
trace the path $\pi$ which connects the start $S$ with the target $T$ 
(Line~5).  
If such a path is found and is better than $\pi_{\textup{best}}$, it replaces $\pi_{\textup{best}}$ (Lines~6,~7).  
$\pi_{\textup{best}}$ is returned after the time limit is reached (Line~8).

\begin{algorithm}[ht]
\caption{$\udrrtstar {\tt(} \mmgimp, S, T, \nit {\tt)}$}
\label{algo:udrrtstar}
$\pi_{\textup{best}} \gets \emptyset$,
$\tree.{\tt init}(S)$\;
\While{${\tt time.elapsed}() < {\tt time\_limit}$}
{
    
    \For{$i : 1 \to \nit$}
    {
        ${\tt Expand\_\udrrtstar(} \mmgimp, \tree {\tt)}$\;
    }
    $\pi \gets {\tt Connect\_to\_Target(} \mmgimp, \tree, T {\tt)}$\;
    \If{$\pi \neq \emptyset\ \cap\ \cost(\pi) < \cost(\pi_{\textup{best}})$}
    {
        $\pi_{\textup{best}} \gets {\tt Trace\_Path(} \tree, T {\tt)}$
    }
}
{\bf return $\pi_{\textnormal{best}}$}
\end{algorithm}
 \begin{algorithm}[ht]
\caption{${\tt Expand\_\udrrtstar(} \mmgimp, \tree{\tt)}$}
\label{algo:udrrtstar_expand}
$\qrand \gets {\tt Random\_Sample()}$\;
$\vnear \gets {\tt Nearest\_Neighbor(} \tree, \qrand {\tt)}$\;
$\vnew \gets {\tt \oracle (} \vnear, \qrand {\tt)}$\;
\If{ $\vnew \notin \tree$ }
{
    $\tree.{\tt Add\_Vertex(} \vnew {\tt)}$\;
    $\tree.{\tt Add\_Edge(} \vnear, \vnew {\tt)}$\;
}
\Else
{
    $\tree.{\tt Rewire(} \vnear, \vnew {\tt)}$\;    
}
\end{algorithm}
 \begin{algorithm}[ht]
\caption{${\tt \oracle (} \vnear, \qrand, \mmgimp {\tt)}$}
\label{algo:drrtstar_oracle}

\For{$ i : 1 \rightarrow R $}
{
	{
		$ \nodenew_i \leftarrow \underset{v\in{\tt Adj}(\nodenear_i,\graph_i)}{argmin} {\angle_{\nodenear_i}} (\noderand_i,v) $ \; 
	}
}

\textbf{return} $ \vnew $

\end{algorithm}

The expansion procedure for $\udrrtstar$ is very similar to the
original $\drrt$ method, and is outlined in
Algorithm~\ref{algo:udrrtstar_expand}.  It begins
by drawing a random sample in the composite configuration space
(Line~1), and then finds the nearest neighbor $\vnear$ to this sample
in the tree (Line~2). It then selects a neighbor $\vnew$ according to
the oracle function $ \oracle $ (Algorithm~\ref{algo:drrtstar_oracle}).
This is the same oracle that is used in $ \drrt $ that tries to select
 a neighbor of $ \vnear $ most in the direction of $ \qrand $.  
 Then, if $\vnew$ is not in
the tree (Line~4), it is added to the tree (Line~5) and an edge from
$\vnear$ to $\vnew$ is also added (Line~6).  Where this expansion step
differs is that if $\vnew$ is already in the tree (Line~7), the method
performs a rewiring step (Line~8) to \commentadd{check} to see
if the path to $\vnew$ is of lower cost than the existing
one.

The method would be similar to $ \drrt $ in terms of the samples that 
constitute the tree, however $ \udrrtstar $ improves the solution cost 
with iterations and finds better solutions compared to $ \drrt $. It is however desirable to focus the search in order find the initial solution quickly, while preserving solution quality improvement over time.

\subsection{$\drrtstar$}
The main body of the informed \drrtstar\ algorithm is provided in 
Algorithm~\ref{algo:drrtstar}. The proposed method is an improvement on top of $ \udrrtstar $, that preserves the asymptotic optimality while benefiting computationally from \textit{branch-and-bound} pruning once a solution is found, greedy child propagations during node selection, and heuristic guidance during expansions. 

\begin{figure}[h]
	\centering
	\includegraphics[width=3in]{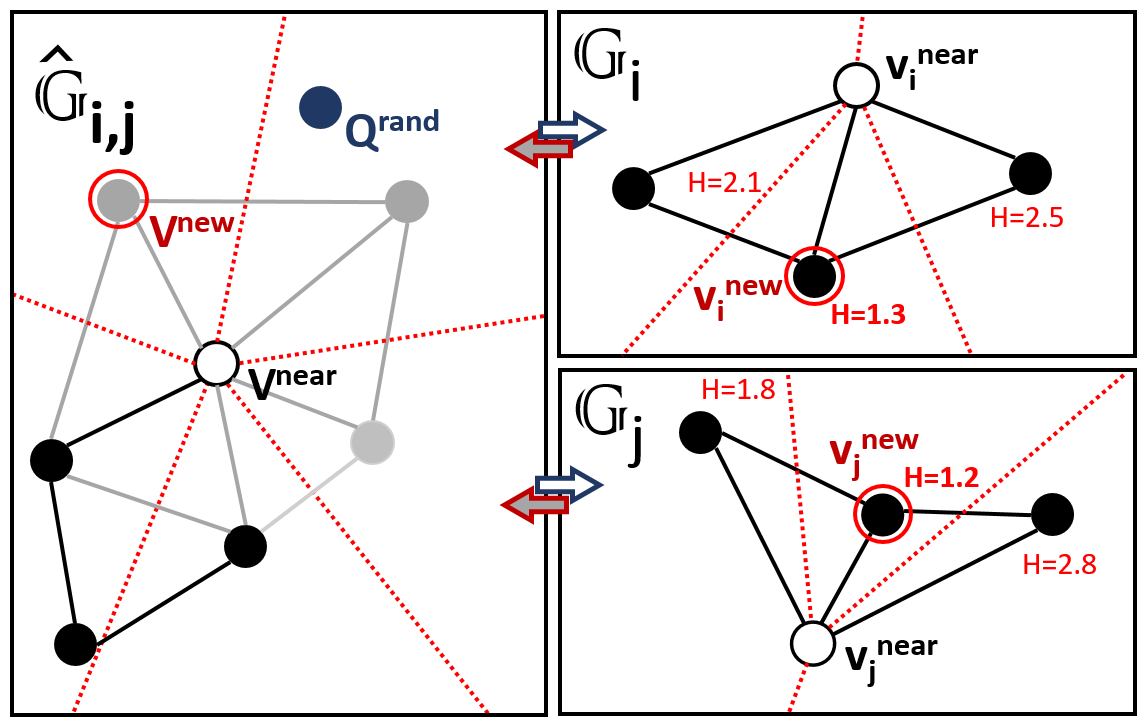}
		\caption{The method selects a neighbor with the best heuristic estimate  $ \heuristic $ from each of the constituent roadmaps during the guided expansion phase in $ \ioracle $.}
	\label{fig:dadrrt_expansion}
\end{figure}

The key insight behind the algorithmic improvements is the fact that by virtue of the structure of the \textit{tensor-product roadmap} $ \mmgimp $, there readily exists a usable heuristic measure $ \heuristic $ in the constituent roadmaps $ \mmgraph_i $. The shortest path on a constituent roadmap to the goal $ T $ can be used as a heuristic to guide the tree. 
If there is no robot interaction introduced by the individual robot shortest paths, such a path comprising of the individual shortest paths is a solution that suffices.
In cases of interaction between the robots, a shortest path is expected to deviate locally in regions of interaction. The best a robot can do from any constituent roadmap vertex is to follow the shortest path to the goal on the constituent roadmap. Although domain-specific heuristics can be also applied to the algorithm, it should be noted that in the currently proposed method, the purpose of the heuristic is to primarily discover an initial solution as quickly as possible. This in turn helps \textit{branch-and-bound} kick in and further focuses the search once a bounding cost is ascertained from the initial solution.

At a high level, Algorithm~\ref{algo:drrtstar} follows the structure of $ \udrrtstar $. The only change is that the outer loop keeps track of the tree node being added as $ \vlast $ and passes it on to the next call to the $ \mathtt{Expand\_\drrtstar} $ subroutine. The use of this information to apply heuristic guidance is detailed in the description of the function.

Algorithm~\ref{algo:drrtstar_expand} outlines the expansion step. The
default behavior is summarized in Algorithm~\ref{algo:drrtstar_expand}, Lines~1-3, i.e., when no $\vlast$ is
passed as argument (Line~1). This operation corresponds to an
exploration step similar to \rrt, i.e., a random sample $\qrand$
is generated in $\ccross$ (Line~2) and its nearest neighbor $\vnear$
in $\tree$ is found (Line~3). 
If the last iteration generated a node $\vlast$ that was closer to the goal 
relative to its parent, then $\vlast$ is provided to the function. In this case (Line~4-6) $\qrand$ is set as the target 
configuration $ T $, and $\vnear$ is selected as $ \vlast $.
This constitutes the greedy child propagations which promotes progress towards the goal.

\begin{algorithm}[!ht]
\caption{$\drrtstar {\tt(} \mmgimp, S, T, \nit{\tt)}$}
\label{algo:drrtstar}
$\pi_{\textup{best}} \gets \emptyset$,
$\tree.{\tt init}(S), \vlast \gets S$\;
\While{${\tt time.elapsed}() < {\tt time\_limit}$}
{       
        
    \For{$i : 1 \to \nit$}
    {
        $\vlast \gets {\tt Expand\_\drrtstar(} \mmgimp , \tree, \vlast, T {\tt)}$\;
    }
    $\pi \gets {\tt Connect\_to\_Target(} \mmgimp, \tree, T {\tt)}$\;
    \If{$\pi \neq \emptyset\ \cap\ \cost(\pi) < \cost(\pi_{\textup{best}})$}
    {
        $\pi_{\textup{best}} \gets {\tt Trace\_Path(} \tree, T {\tt)}$
    }
}
{\bf return $\pi_{\textnormal{best}}$}
\end{algorithm}
 
\begin{algorithm}[ht]

\caption{${\tt Expand\_\drrtstar(} \mmgimp, \tree, \vlast, T {\tt)}$}
\label{algo:drrtstar_expand}
\If{$\vlast = \emptyset$}
{
    $\qrand \gets {\tt Random\_Sample()}$\;
    $\vnear \gets {\tt Nearest\_Neighbor(} \tree, \qrand {\tt)}$\;
    
}
\Else
{

    $\qrand \gets T$\;
    $\vnear \gets \vlast$\;
    
}
$\vnew \gets {\tt \ioracle (} \vnear, \qrand, \mmgimp, T {\tt)}$\;
$N \gets {\tt Adj(} \vnew, \mmgimp {\tt )} \ \cap \ \tree$\;

$\vparent \gets \underset{V \in N }{argmin} {\ \ \cost(V) + \cost( \local(V, \vnew ))} $\;

\lIf{$\vparent = \emptyset$}
{
    ${\bf return}\ \emptyset$
}
\lIf{$\cost(\vnew)>\cost(\pi_{\textup{best}}) $}
{
    ${\bf return}\ \emptyset$
}
\If{ $\vnew \notin \tree$ }
{
$\tree.{\tt Add\_Vertex(} \vnew {\tt)}$\;
$\tree.{\tt Add\_Edge(} \vparent, \vnew{\tt)}$\;
}
\lElse
{
    $\tree.{\tt Rewire(} \vparent, \vnew {\tt)}$   
}
\For{ $V \in N$ }
{
    \If{$\cost(\vnew)+\cost(\local(\vnew,V))<\cost(V)\ \cap\  \local(\vnew,V)\subset\cfree$}
    {
        $\tree.{\tt Rewire(} \vnew, V {\tt )}$\;
    }
}
\If{${\tt \heuristic}(\vnew) < {\tt \heuristic}(\vparent)$ }
{
  ${\bf return}\ \vnew$\;
}
\lElse
{
    ${\bf return}\ \emptyset$
}
\end{algorithm}

 \begin{algorithm}[!h]
\caption{${\tt \ioracle (} \vnear, \qrand, \mmgimp, T {\tt)}$}
\label{algo:idrrtstar_oracle}

\For{$ i : 1 \rightarrow R $}
{
	\If{ $ \noderand_i = T_i $ }
	{
		$ \nodenew_i \leftarrow \underset{x\in{\tt Adj}(\nodenear_i,\graph_i)}{argmin}{\tt \heuristic(} x, T_i, \graph_i {\tt )} $ \; 
	}
	\Else
	{
		$ \nodenew_i  \leftarrow {\tt random}({\tt Adj}(\nodenear_i,\graph_i))$\;
	}
}

\textbf{return} $ \vnew $

\end{algorithm}

 \textbf{Informed Expansion $ \ioracle $}: The expansion procedure in Algorithm~\ref{algo:idrrtstar_oracle} replaces the oracle in $ \drrt $. It switches between distinct guided and exploratory behaviors according to whether $ q^{rand}_i $ attempts to drive the expansion towards the target $ T $ or not. When the method uses heuristic guidance, among all the neighbors of $ \nodenear_i $ on a constituent roadmap $ \mmgraph_i $, $ \mathtt{Adj}(\nodenear_i, \mmgraph_i) $, the one with the best heuristic measure $ \heuristic $ is selected. During the exploration phase, the method selects a random neighbor out of $ \mathtt{Adj}(\nodenear_i, \mmgraph_i) $.

In either case, the oracle function $\ioracle$ returns to Algorithm~\ref{algo:drrtstar_expand} the implicit graph node \vnew\ that is a neighbor of 
\vnear\ on the implicit graph (Line~7).  
Then the method finds neighbors $N$, which are adjacent to \vnew\ in 
$\mmgimp$ and have also been added to $\tree$ (Line~8).  Among $N$, 
the best node $\vparent$ is chosen as the node to connect $ \vnew $ according to cost measure. Such an operation might yield no valid parent $ \vparent $ due to collisions along $ \mathbb{L}(\cdot) $. In such a case (Line~10) the method fails to add a node during the current iteration. Line~11 implements a \textit{branch-and-bound} based on the cost of the best solution so far.

\commentadd{Lines~12-15 recount} the tree addition and parent rewire process. \commentadd{Lines~16-17 perform} an additional rewiring step in the neighborhood if $ \vnew $ is a better parent of any of the neighboring nodes. \commentadd{Line 19 switches} the child promotion by checking whether $ \vnew $ made progress toward the goal according to the heuristic measure. The method ensures that in this case $ \vnew $(or $ \vlast $) is a \commentadd{child-promoted node, which} would be selected during the next iteration. \commentadd{This effect of this behavior} is that if the uncoordinated individual shortest paths are collision free, this would be greedily attempted first from the child-promoted nodes added to the tree. Evaluations indicate that this proves very effective in practice.

It should be noted that all candidate edges $ \mathbb{L}(\cdot) $ in Line 9 and 17 are collision-free and for the sake of algorithmic clarity, collision checking has been assumed to be encoded into the steering function $ \mathbb{L}(\cdot) $ and this is enforced during tree additions and rewires. Any specialized sampling behavior is assumed to be part of the implementation of the subroutine $ \mathtt{Random\_Sample} $.

\noindent \textbf{Notes on Implementation}: In the implementation, the
heuristic measure $ \heuristic $ is efficiently calculated by
precomputing all-pair shortest paths on the constituent graphs with
\textit{Johnson's} algorithm \citep{johnson1977efficient}, which runs
in $ \mathcal{O}(|V|^2log|V|) $. Precomputing the heuristic measure
alleviates any overhead of spending online computation time. It is
proposed that for large graphs, the vertices can be subsampled and the
heuristic estimated for representative nodes that approximate the
$\heuristic$ value in their neighborhoods. The neighbor with the best
$\heuristic$ value can be computed once for a given target $T$, and
reused during the iterations inside $\ioracle$. In
Algorithm~\ref{algo:drrtstar_expand} (Line 19), $ \heuristic $ refers
to a heuristic estimate in the composite space. This can be deduced
from the constituent spaces. The method also included additional
focused random sampling \citep{gammell2015batch} once a solution is
found to aid in convergence. \commentadd{For a fraction} of the
``random samples'', goal biasing samples the target state for a
robot.

The set $ \mathtt{Adj}(\cdot) $ returns the set of neighbors on the
graph for a vertex, in addition to the vertex itself. This ensures
that it is possible for a robot to stay static during an edge
expansion. \commentadd{This means that the algorithm is also able to
  discover solutions where a subset of the robots must remain
  stationary for a period of time.}

\section{Analysis}
\label{sec:analysis}

This section examines the properties of \drrtstar\ starting with the
asymptotic convergence of the implicit roadmap $\mmgimp$ to containing
a path in $\cfree$ with optimum cost.  Then, it is shown that
the online search eventually discovers the shortest path in $\mmgimp$. The
combination of these two facts proves the asymptotic optimality of
\drrtstar and \udrrtstar. 

For simplicity, the analysis considers robots operating in Euclidean
space, i.e., $\cspace_i$ is a $d$-dimensional Euclidean hypercube
$[0,1]^d$ for fixed $d \geq 2$. Robots are assumed to have the
same number of degrees of freedom $d$. \commentadd{The results can relate to a large class of systems, which
  are locally Euclidean (see, \cite{dobson2013study}). This is
  applicable to all the systems under consideration in the paper,
  including manipulators, with bounded angular degrees of
  freedom. Analysis of systems, which are not locally Euclidean,
  requires additional rigor especially regarding the definition of the
  cost metric. The Discussion section includes a description of a
  possible extension of the presented analysis to non-holonomic
  systems. It is acknowledged that the arguments presented in the
  current section will not readily transfer to such systems.}

\subsection{Optimal Convergence of $\mmgimp$}
In this section we prove that when the connection radius $r(n)$
\commentadd{\footnote{In the graphs considered here, an edge exists
    between two nodes, if the nodes are separated by a distance less
    than the connection radius $r(n)$.}} used for the construction of
the single-robot \prm\ roadmaps $\graph_1,\ldots,\graph_R$ is chosen
in a certain manner, this yields a tensor-product graph $\mmgimp$,
which \commentadd{contains} asymptotically optimal paths for MMP.

\begin{definition}
  A trajectory $\Sigma:[0,1]\rightarrow \cfree$ is \emph{robust} if
  there exists a fixed $\delta>0$ such that for every
  $\tau\in [0,1],X\in \cinv$ it holds that
  $\|\Sigma(\tau)-X\|\geq \delta$, where $\|\cdot\|$ denotes the
  standard Euclidean distance.
\end{definition}

\begin{definition}
  Let $\cost$ \commentadd{be one} of the cost functions defined in
  Section~\ref{sec:setup}.  A value $c>0$ is \emph{robust} if for
  every fixed $\epsilon>0$, there exists a robust path $\Sigma$, such
  that $\cost(\Sigma)\leq(1+\epsilon)c$. The \emph{robust optimum}
  $c^*$, is the infimum over all such values.
\end{definition}

For any fixed $n\in \mathbb{N}^+$, and a specific instance of
$\mmgimp$ constructed from $R$ roadmaps, having $n$ samples each,
denote by $\Sigma^{(n)}$ the lowest-cost path (with respect to
$\cost(\cdot)$) from $S$ to $T$ over~$\mmgimp$.

\begin{definition}
  $\mmgimp$ is asymptotically optimal (AO) if for every fixed
  $\epsilon > 0$ it holds that
  $\cost(\Sigma^{(n)}) \leq (1+\epsilon)c^*$ \commentadd{asymptotically almost
  surely}\footnote{Let $A_1,A_2\ldots$ be random variables in some
    probability space and let $B$ be an event depending on~$A_n$. $B$
    occurs \emph{asymptotically almost surely} (a.a.s.) if
    $\lim\limits_{n\rightarrow \infty} \Pr[B(A_n)]=1$.}, where the probability
  is over all the instantiations of $\mmgimp$ with $n$ samples for
  each {\tt PRM}.
\end{definition}

Using this definition, the following theorem is proven.  Recall that 
$d$ denotes the dimension of a single-robot configuration space. 
 
\begin{theorem}
$\mmgimp$ is AO when $$r(n)\geq \radstar= \gamma \left( \frac{\log n}{n} \right)^{\frac{1}{d}},$$
where $ \gamma = (1+\eta)2 \left(\frac{1}{d} \right)^{\frac{1}{d}} \left(  \frac{\mu(\cfree)}{\zeta_d} \right)^{\frac{1}{d}}$
where $\eta$ is any constant larger than $0$, $ \mu $ is the volume measure and $ \zeta_d $ is the volume of an unit hyperball in $ \mathbb{R}^d $.
\label{thm:opt_graph}
\end{theorem}

Since the method deals with solving the problem of finding a \textit{robust optimum} solution, some $\epsilon>0$ is fixed. 
By definition of the problem, there exists a robust trajectory
$\Sigma:[0,1] \rightarrow \cfree$, and a fixed $\delta>0$, such that
$\cost(\Sigma)\leq (1+\epsilon) c^*$. Additionally for every
$X \in \cinv,\ \tau\in [0,1]$ it holds that
$\|\Sigma(\tau)-X\|\geq \delta$.  

If it can be \commentadd{shown that} $\mmgimp$ contains a trajectory $\Sigma^{(n)}$, such
that\footnote{The small-o notation $o(1)$ indicates a function that
  becomes smaller than any positive constant and thereby
  asymptotically will become negligible. When this relation holds, the
  positive constant corresponds to $\epsilon$.}:
\begin{align}
\hspace{0.5in}  \cost(\Sigma^{(n)})\leq (1+o(1))\cdot \cost(\Sigma) 
 \label{eq:eps_approx}
\end{align}
a.a.s., this would imply that
$\cost(\Sigma^{(n)}) \leq (1+\epsilon)c^*$, proving
Theorem~\ref{thm:opt_graph}.

As a first step, it
will be shown that the robustness of
$\Sigma = (\sigma_1,\ldots,\sigma_R)$ in the composite space implies
robustness in the single-robot setting, i.e., robustness along
$\sigma_i$.

For $\tau \in [0,1]$ define the forbidden space parameterized by
$\tau$ as  
\begin{equation}
\hspace{0.5in} \cinv_i(\tau) = \cinv_i \cup \bigcup_{j=1,
  j \neq i}^R I_i^j( \sigma_j (\tau) ).
\end{equation}

\begin{claimthm}
For every robot $i$, $\tau \in [0,1]$, and $q_i \in \cinv_i(\tau)$, 
$\| \sigma_i(\tau) - q_i \| \geq \delta$, \commentadd{i.e., the robustness of
$\Sigma = (\sigma_1,\ldots,\sigma_R)$ in the composite space implies
robustness over all single-robot paths $\sigma_i$.}
\label{claim:robust}
\end{claimthm}
\begin{proof}
  Fix a robot $i$, and fix some $\tau \in [0,1]$ and a 
  configuration $q_i \in \cinv_i(\tau)$.  Next, define the
  following composite configuration
$$Q = (\sigma^1(\tau), \dots, q_i, \dots , \sigma^R(\tau)).$$
Note that it differs from $\Sigma(\tau)$ only in the $i$-th robot's
configuration. By the robustness of $\Sigma$ it follows that

\begin{align*}
\delta & \leq \| \Sigma(\tau) - Q \| \\ 
       & = \bigg( \| \sigma_i(\tau) - q_i \|^2 + \sum_{j=1,j \neq i}^R \| \sigma_j(\tau) -
           \sigma_j(\tau) \|^2 \bigg)^{\frac{1}{2}} \\ 
       &\leq \| \sigma_i(\tau) - q_i \|. 
\end{align*}
\end{proof}

The result of Claim \ref{claim:robust} is that the paths
$\sigma_1, \dots, \sigma_R$ are robust in their individual spaces w.r.t the parameterized forbidden space $ \cinv_i(\tau) $. This means that there is
sufficient clearance for the individual robots to not collide with
each other given a fixed location of a single robot.  

Next, a Lemma is
derived using proof techniques from the
literature~\cite[Theorem~4.1]{Pavone:2015fmt}, and it implies every
$\graph_i$ contains a single-robot path $\sigma_i^{(n)}$ that
converges to $\sigma_i$.
\begin{lemma}
For every robot $i$, let $\graph_i$ be constructed with $n$ samples and a
connection radius $r(n)\geq \radstar$. \commentadd{Then it} contains a 
path $\sigma_i^{(n)}$ with the following attributes a.a.s.: 
\begin{itemize}
\item[(i)] $\sigma_i^{(n)}(0) = s_i$, $\sigma_i^{(n)}(1) = t_i$; 
\item[(ii)] $\|\sigma_i^{(n)}\| \leq (1 + o(1)) \|\sigma_i\|$; 
\item[(iii)] $\forall q \in \textup{Im}(\sigma_i^{(n)})$, $\exists 
\tau \in [0,1]$ s.t. $\|q-\sigma_i(\tau)\| \leq \radstar$, where
\textup{Im}$(\cdot)$ \commentadd{is the function image}.
\end{itemize}
\label{lem:prm}
\end{lemma}

\begin{proof}	

\commentadd{The first two properties of Lemma (i) and (ii) restate
  ~\cite[Theorem~4.1]{Pavone:2015fmt}, which is applicable to the
  setup of this work. The last property (iii) is an immediate
  corollary of the first two: due to the fact that $\sigma_i^{(n)}$ is
  obtained from $\graph_i$, every point along the path is either a
  vertex of the graph, or lies on a straight-line path (i.e., an edge)
  between two vertices, whose length is at most $\radstar$.}
\end{proof}

To complete the proof of Theorem~\ref{thm:opt_graph} it remains to be
shown that the combination of $\sigma_1^{(n)},\ldots, \sigma_R^{(n)}$
yields \commentadd{the trajectory}, $ \Sigma^{(n)} $ of a desired cost, i.e., one that conforms to Equation~\ref{eq:eps_approx}. The bound derived in Lemma~\ref{lem:prm} (ii) looks like what we need for proving Theorem~\ref{thm:opt_graph}. Even though a similar bound exists in the individual spaces, it needs to be shown that Equation~\ref{eq:eps_approx} holds for a cost function in $ \cfree $. We proceed to show this
individually for the different cost functions.

\subsubsection{Optimal Convergence for a Linear combination of Euclidean arc lengths}

\begin{lemma}
Given Lemma~\ref{lem:prm} (ii), Equation~\ref{eq:eps_approx} holds for a cost function $ \cost(\cdot) $ that is a linear combination of Euclidean arc lengths
\label{lem:sumofparts}
\end{lemma}
\begin{proof}
Here consider the case that $\cost(\Sigma)= \sum_{i=1}^R \|\sigma_i\|$, which can also be easily modified for $\max_{i=1:R} \|\sigma_i\|$ or some arbitrary linear combination of the arc lengths. 

In particular, define
$\Sigma^{(n)}=(\sigma_1^{(n)},\ldots, \sigma_R^{(n)})$, where
$\sigma_i^{(n)}$ are obtained from Lemma~\ref{lem:prm}. Then
\begin{align*}
\cost(\Sigma^{(n)})=\sum_{i=1}^R\|\sigma_i^{(n)}\|&\leq (1 +
o(1))\sum_{i=1}^R\|\sigma_i\|\\
&\leq (1+o(1))\cost(\Sigma).
\end{align*}

\end{proof}

\begin{lemma}
A path $ \Sigma^{(n)} = (\sigma_1^{(n)}\ldots\sigma_R^{(n)}) $ exists, that satisfies the properties of Lemma~\ref{lem:prm}, and is collision free both in terms of robot-obstacle and robot-robot.
\label{lem:coordination}
\end{lemma}
\begin{proof}
\commentadd{Every constituent roadmap $\graph_i$ of $n$ samples is
  constructed to satisfy Lemma~\ref{lem:prm} and contains individual
  robot paths $\sigma_i^{(n)}$.  $\mmgimp$ defines the tensor-product
  graph in the composite configuration space $\cfull$. The path $
  \Sigma^{(n)}$ is a combination of the individual robot paths
  $\sigma_i^{(n)}$.  Lemma~\ref{lem:prm} implies that $\mmgimp$
  contains a path $ \Sigma^{(n)}$ in $\cfull$, that represents
  collision-free motions relative to obstacles, and minimizes the cost
  function.}  Nevertheless, it is not clear whether this ensures the
existence of a path where robot-robot collisions are avoided.  That
is, although $\textup{Im}(\sigma^{(n)}_i)\subset \cfree_i$, it might
be the case that $\textup{Im}(\Sigma^{(n)})\cap \cinv \neq \emptyset$.
Next, it is shown that $\sigma_1^{(n)},\ldots, \sigma_R^{(n)}$ can be
reparametrized to induce a composite-space path whose image is fully
contained in $\cfree$, with length equivalent to $\Sigma^{(n)}$.

For each robot $i$, denote by $V_i=(v_i^1,\ldots,v_i^{\ell_i})$ the
chain of $\graph_i$ vertices traversed by $\sigma^{(n)}_i$. For every
$v_i^j\in V_i$ assign a timestamp $\tau_i^j$ of the closest
configuration along \commentadd{$\sigma_i$}, i.e., 
$$\tau_i^j=\argmin_{\tau\in [0,1]}\|v_i^j-\sigma_i(\tau)\|.$$
Also, define $\T_i=(\tau_i^1,\ldots,\tau_i^{\ell_i})$ and denote by
$\T$ the ordered list of $\bigcup_{i=1}^R\T_i$, according to the
timestamp values. Now, for every $i$, define a global timestamp
function $T\!S_i:\T \rightarrow V_i$, which assigns to each global 
timestamp in $\T$ a single-robot configuration from $V_i$.  It thus
specifies in which vertex robot $i$ resides at time $\tau \in \T$.
For $\tau\in \T$, let $j$ be the largest index, such that
$\tau_i^j \leq \tau$. Then simply assign $T\!S_i(\tau)= \tau_i^j$. 
From property (iii) in Lemma~\ref{lem:prm} and 
Claim~\ref{claim:robust} it follows that no robot-robot collisions are 
induced by the reparametrization. This concludes the proof of
Theorem~\ref{thm:opt_graph}.
\end{proof}

\subsubsection{Optimal Convergence for Euclidean arc length}
Arguments for convergence of the cost of the solution in terms of the Euclidean arc length of the composite path $ \Sigma $ can be made to extend the results of Lemma~\ref{lem:sumofparts}. A robot having $ d $ \dof s, $ (F_1, \dots F_d) $ exists in an $\mathbb{R}^d$ space. The motion of the robot constitutes a curve in $\mathbb{R}^d$, defined as 
$$ \Sigma = (f_1(t), \dots f_d(t)),$$ 
where $f_i(t)$ is the coordinate function of the curve $\Sigma$ along the \dof\ $F_i$, where $t\in[0,1]$ and $i \in [1,d]$.
The \commentadd{\textit{Euclidean} arc-length of the path} in
$\mathbb{R}^d$:
\begin{equation}\label{eq:metric_definition}
\|\Sigma\| = \int { \sqrt{(f'_1(t)^2 + \dots + f'_d(t)^2)}}\ dt.
\end{equation}

 The coordinate function is assumed to be continuous with respect to
 the Lebesgue measure on $t$ and of bounded variation. \commentadd{The
 Lebesgue measure assigns a measure to subsets of $n$-dimensional
 Euclidean space. For $n = 1, 2,$ or $3$, it coincides with the
 standard measure of length, area, or volume. The assumption here is
 that the curve is Lebesgue integrable over the coordinate function.
 This relates to the variation of the curve being smooth for the
 parametrization, over subsets of the curve that correspond to the
 Lebesgue measure over $t$.} $ \Sigma $ is also a rectifiable curve,
 i.e., the curve has finite length.

\begin{figure}[!t]
\centering \includegraphics[width=0.4\textwidth]{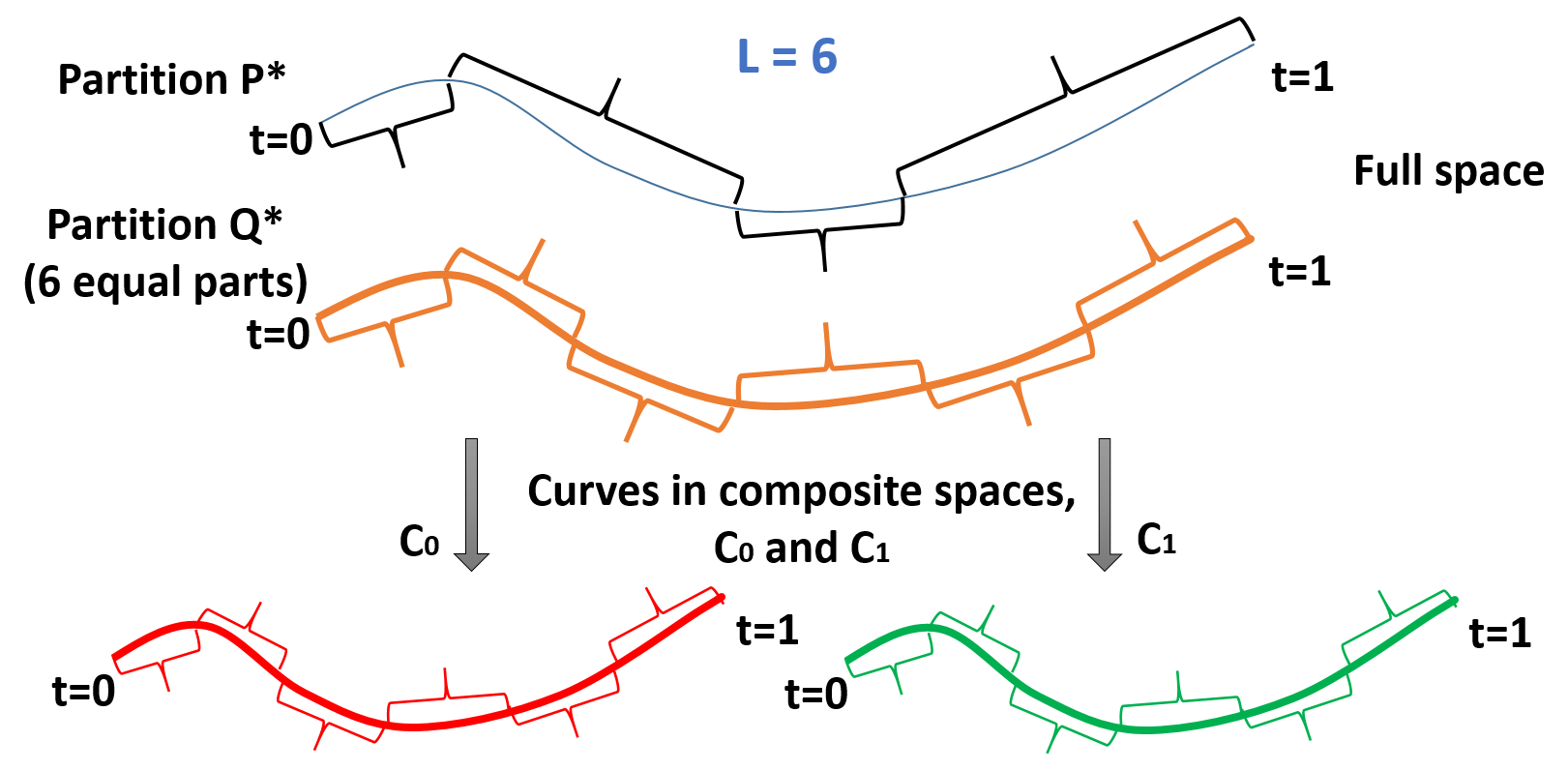}
\caption{Partitions of the curves in the full space and the
	composite \commentadd{spaces (bottom)}. $ P^* $ is the
	supremum partition but the parts do not have equal measures. $
	P^* $ is subdivided into a finer partition $ Q $ with $ L=6 $
	equal measure parts. \commentadd{Let the composite curves
	exist in individual robot configuration spaces $ \cspace_0 $
	and $ \cspace_1 $, respectively. These curves} also have a
	finer partition with $ 6 $ equal parts.}
\label{fig:euclidean_partitions}
\end{figure}
\begin{definition}
	
The partial arc length is defined as, 
$$ s(x) = \|\Sigma:[0,x]\|, x\leq1.$$
By definition and assuming smoothness and bounded variation \citep{pelling77}, for $ t=0
$ to $ t=x \leq 1:$
\begin{equation}\label{euclidean}
s(x) = \sup_{P} \sum_{j=1}^{m} \sqrt{\bigg( \sum_{i=1}^{d} ( f_i(t_j) - f_i(t_{j-1} ) )^2 \bigg)}.
\end{equation}
The value is the supremum over all possible finite partitions $ P:0=t_0<\dots<t_m=x $, that can divide $ t $. This generates a finite set of $ m $ parts. 
\end{definition}

We denote the value of $ s(x) $ for some partition $ P $ as:
\begin{equation}\label{partition}
s(x)_P = \sum_{j=1}^{m} \sqrt{ \sum_{i=1}^{d} ( f_i(P(j)) - f_i(P({j-1}) ) )^2 }.
\end{equation}

 A \textit{part} shall refer to the curve between $ \Sigma(P(j-1)) $ and $ \Sigma(P(j)) $. The measure of the part would be 
\begin{equation}\label{measure}
\mu_{j-1,j}^P = s(P(j)) - s(P(j-1)).
\end{equation}

Let $ P^* $ be the \textit{finite} supremum partitioning over $ t $,
that has $ m $ parts. This means, $ s(x) = s(x)_{P^*} $. 
Without loss of generality, let us assume that $ P^* $ corresponds to
the supremum partitioning that has the least number of parts, $
|P^*|=m+1 $, \commentadd{i.e.,} there are no degenerate partitions. A
finer partition can introduce additional parameterization over $ t
$, \commentadd{and hence is a superset of $ P^* $,} but cannot
increase the value of $ s(x) $ since $ P^* $ is the
supremum. \commentadd{Note that a partition sequence with $m$ parts
will have a cardinality of $m+1$. The finer partition has all these
$m+1$ parametrization values in addition to others.}. Since $ s(x) $
is finite and $ m $ is finite, $ \exists P^* \
| \ \mu_{j-1,j}^{P^*} \in \mathbb{R}_{+} \ \forall \ j \in t $.

\begin{claimthm}
Given a finite set of paths $ \xi $, there exists a finer partitioning, $ Q^* $ for $ P^* $ over each $ \Sigma\in\xi $, which yields $ L $ number of equal measure parts (Figure \ref{fig:euclidean_partitions}) for every $ \Sigma\in\xi $.
\label{claim:partition}
\end{claimthm}
\begin{proof}
\begin{align*}
\exists L \ \ \text{s.t.} \
\ s(1)_{Q^*} = s(1)_{P^{*}} = \|\Sigma\|, \\ 
Q^* \supseteq P^*,\ |Q^*| = L+1,\\ 
\mu_{j-1,j} = \mu_{k-1,k} = \frac{\|\Sigma\|}{L} \in \mathbb{R}_{+} \ 
\ \forall \ j,k \in [1,\ldots L].
\end{align*}
This holds true for every $ \Sigma $ for a corresponding $ Q^* \supseteq P^* $.
The measure of every \textit{part} $ \Sigma(l) $ is equal, and is denoted by
$ \| \Sigma(l) \|
= \frac{\|\Sigma \|}{L} $. This simplifies Equation~\ref{partition} to $ \| \Sigma \| = \sum_{l=1}^{L} \| \Sigma(l) \|$.

\end{proof}

\begin{claimthm}
\label{claim:recombination}
Additionally, by this simplification, Equation~\ref{partition} in the composite space is restated for \commentadd{$ \Sigma_{Rd}=\{\Sigma_1\ldots\Sigma_R\} $} where $ \Sigma_{Rd}, \Sigma_1\ldots\Sigma_R\in\xi $.
\end{claimthm}

\begin{proof}
\commentadd{In the multi-robot space Euclidean space $\mathbb{R}^{Rd}$, the arc length in the composite space can be expressed in terms of the arc lengths traversed in the individual robot spaces.}
\begin{align*}
\| \Sigma_{Rd} \| &= s(1)_{Q^*} \\
&= \sum_{l=1}^{L} \sqrt{ \sum_{i=1}^{Rd} ( f_i(Q^*(j)) - f_i(Q^*({j-1}) ) )^2 } \\
&= \sum_{l=1}^{L} \sqrt{ \sum_{i=1}^{d} ( \delta f_i )^2 + \ldots + \sum_{i=(R-1)d+1}^{Rd} ( \delta f_i )^2} \\
&= \sum_{l=1}^{L} \sqrt{ \| \Sigma_1(l) \|^2 + \ldots + \| \Sigma_R(l) \|^2} \\
&= \sum_{l=1}^{L} \sqrt{ \sum_{i=1}^{R} \| \Sigma_i(l) \|^2 },
\end{align*}
where, with a slight abuse of notation $ \delta f_i $ is a shorthand representation for some $ f_i(Q^*(j)) - f_i(Q^*({j-1}) ) $.

\end{proof}

\begin{lemma}\label{lem:convergence}
	For a $\Sigma^{(n)} = (\sigma_1^{(n)}\ldots \sigma_R^{(n)})$, where
	$\sigma_i^{(n)}$ is obtained from Lemma~\ref{lem:prm}, given that $\|\sigma_i^{(n)}\| \leq (1 + o(1)) \|\sigma_i\|$, Equation~\ref{eq:eps_approx} holds for the Euclidean arc lengths.	
\end{lemma}
\begin{proof}
Partitioning the arcs $ \sigma_i^{(n)} $, and $ \sigma_i $,
into \commentadd{$ L $ (chosen as} per Claim~\ref{claim:partition}) pieces of equal length, yields two trajectory sequences,
for $ \ l \in N_+, l \leq L $.

The high level idea is that leveraging the uniformity in the parameterized \textit{parts} introduced by $ L $, Lemma~\ref{lem:prm}(ii) has to be recombined to represent the Euclidean arc length in the composite space.

\begin{align*}
&\|\sigma_i^{(n)}\| \leq (1 + o(1)) \|\sigma_i\| \tag{Using Lemma~\ref{lem:prm}}\\
\Rightarrow &\sum_{l=1}^L \| \sigma_i^{(n)}(l) \| \leq (1 + o(1))  \sum_{l=1}^L \| \sigma_i(l) \| \\
\Rightarrow &\| \sigma_i^{(n)}(l) \| \leq (1 + o(1))  \| \sigma_i(l) \| \\
\Rightarrow &\|\sigma_i^{(n)}(l)\|^2 \leq (1 + o(1))^2 \|\sigma_i(l)\|^2. \\
&\text{\noindent Combining over $ R $ using Claim~\ref{claim:recombination},}\\
&\sum_{i=1}^R\|\sigma_i^{(n)}(l)\|^2 \leq  (1 +
o(1))^2\sum_{i=1}^R\|\sigma_i(l)\|^2 \\
\Rightarrow &  L\sqrt{\sum_{i=1}^R\|\sigma_i^{(n)}(l)\|^2  }  \leq  (1 + o(1)) L\sqrt{\sum_{i=1}^R\|\sigma_i(l)\|^2 } \\
\Rightarrow & \|\Sigma^{(n)}\| \leq (1+o(1))\|\Sigma\| \tag{Using Claim~\ref{claim:partition} and~\ref{claim:recombination}}
\end{align*}

\end{proof}

Following the same parameterization described in Lemma~\ref{lem:coordination}, Theorem~\ref{thm:opt_graph} can be shown for the Euclidean metric as well.

\subsection{Asymptotic Optimality of $\drrtstar$}

Finally, \drrtstar\ is shown to be AO.  Denote by $m$ the time budget 
in Algorithm~\ref{algo:drrtstar},  i.e., the number of iterations of 
the loop. Denote by $\Sigma^{(n,m)}$ the solution returned by 
$\drrtstar$ for $n$ samples in the individual constituent roadmaps and $m$ iterations of the \drrtstar\ algorithm.

\begin{theorem}
  If $r(n)>\radstar$ then for every fixed $\epsilon>0$ it holds that
  $$\lim_{n,m\rightarrow \infty}\Pr\left[\cost(\Sigma^{(n,m)})\leq
    (1+\epsilon)c^*\right]=1.$$
  \label{thm:ao_drrt}
\end{theorem}

Since $\mmgimp$ is AO (Theorem~\ref{thm:opt_graph}), it suffices to
show that for any fixed $n$, and a fixed instance of $\mmgimp$,
defined over $R$ {\tt PRM}s with $n$ samples each, $\drrtstar$ 
eventually (as $m$ tends to infinity), finds the optimal trajectory 
over $\mmgimp$.  This property is stated in Lemma~\ref{lem:tree_conv} 
and proven subsequently. The same arguments hold for both $ \drrtstar $ and $ \udrrtstar $, with the difference highlighted explicitly in the proof.

\begin{lemma}[Optimal Tree Convergence of \drrtstar]
\label{lem:tree_conv}
Consider an arbitrary optimal path $\pi^*$ originating from $v_0$ and 
ending at $v_{t}$, then let $O^{(m)}_k$ be the event such that after 
$m$ iterations of \drrtstar, the search tree $\tree$ contains the 
optimal path up to segment $k$.  Then, $$ \liminf_{m \to \infty} \pr 
\big( O^{(m)}_t \big) = 1.$$
\end{lemma}

\begin{proof}
This is shown using Markov chain results 
\cite[Theorem~11.3]{Snell2012:intro_prob}. Specifically, absorbing 
Markov chains can be leveraged to show that $\drrtstar$ will 
eventually contain the optimal path over $\mmgimp$.  An absorbing 
Markov chain has some subset of its states in which the transition
matrix only allows self-transitions.

The proof follows by showing that the $\drrtstar$ method can be 
described as an absorbing Markov chain, where the target state of a
query is represented as an absorbing state in a Markov chain.  For 
completeness, the theorem is re-stated here.

\begin{theorem}[Thm 11.3 in Grinstead \& Snell]
\label{thm:grinstead}
In an absorbing Markov chain, the probability that the process will be 
absorbed is 1 (i.e., $Q(m) \to 0$ as $n \to \infty$), where $Q(m)$ is
the transition submatrix for all non-absorbing states.
\end{theorem}

The first part is that the $\drrtstar$ search is cast as an absorbing 
Markov chain, and second, that the transition probability from each 
state to the next is nonzero, i.e., each state eventually connects to
the target.

For query $(S, T)$, let the sequence $V = \{ v_1, v_2, \dots, 
v_{\textup{t}}\}$ of length $t$ represent the vertices of $\mmgimp$ 
corresponding to the optimal path through the graph which connects 
these points, where $v_{\textup{t}}$ corresponds to the target vertex,
and furthermore, let $v_{\textup{t}}$ be an absorbing state.  
Theorem~\ref{thm:grinstead} works under the assumption that each 
vertex $v_{\textup{i}}$ is connected to an absorbing state $v_{\textup{t}}$.

Then, let the transition probability for each state have two values, 
one for each state transitioning to itself, which corresponds to the
$\drrtstar$ search expanding along some other arbitrary path.  The 
other value is a transition probability from $v_{\textup{i}}$ to 
$v_{\textup{i}+1}$. This corresponds to two slightly different cases for $ \udrrtstar $ and $ \drrtstar $.

\noindent\textbf{Case $ \udrrtstar $}:
The transition probability from $v_{\textup{i}}$ to 
$v_{\textup{i}+1}$ corresponds to the method sampling within
the volume $\textup{Vol}(v_{\textup{i}})$. Then, as the second step, it must be shown that this volume has a 
positive probability of being sampled in each iteration.  It is 
sufficient then to argue that $\frac{\mu(\textup{Vol}
(s_{\textup{i}}))} {\mu(\cfree)} > 0$.  Fortunately, for any finite 
$n$, previous work has already shown that this is the case given 
general position assumptions \cite[Lemma~2]{SoloveySH16:ijrr}.

\noindent\textbf{Case $ \drrtstar $}: In the case of $ \drrtstar $ due to the random neighborhood selection in the expansion $ \ioracle $, there is a positive transition probability from $v_{\textup{i}}$ to 
$v_{\textup{i}+1}$.

Given these results, the $\drrtstar$ is cast as an absorbing Markov
chain, which satisfies the \commentadd{assumptions of Theorem \ref{thm:grinstead}}, and 
therefore, the matrix $Q(m) \to 0$.  This implies that the optimal
path to the goal has been expanded in the tree, and therefore 
$ \liminf_{m \to \infty} \pr \big( O^{(m)}_t \big) = 1.$

\end{proof}

\section{Extension to Shared Degrees of Freedom}
\label{sec:shared}

This section describes an extension of the $ \drrtstar $ approach to systems with shared degrees of freedom (\dof), with specific focus on humanoid robots with two arms. The challenge here arises because of the high dimensionality of the robots. The shared \dof\ is a general formulation, which can refer to either \commentadd{degrees of freedom in a torso or} a mobile base etc.

This section is structured in the same way as the rest of the algorithmic descriptions, and a lot of the shared notations and details are omitted for the sake of brevity. Instead, the interesting insights into the problems that arise due to the shared \dof\ are highlighted, and resolved. A high level overview of the differences of dual-arm $ \drrtstar $ ($ \dadrrtstar $) from the previously stated methods includes:
\begin{itemize}
\item $\dadrrtstar$ decomposes the space by grouping the shared \dof\ with one of the arms.
\item $\dadrrtstar$ implicitly builds two trees online that explores two tensor roadmaps.
\item $\dadrrtstar$ needs additional arguments for proving robustness in Claim \ref{claim:robust}.
\end{itemize}

\begin{wrapfigure}{r}{0.235\textwidth} 	\vspace{-0.4in}
	\centering
	\includegraphics[width=0.235\textwidth]{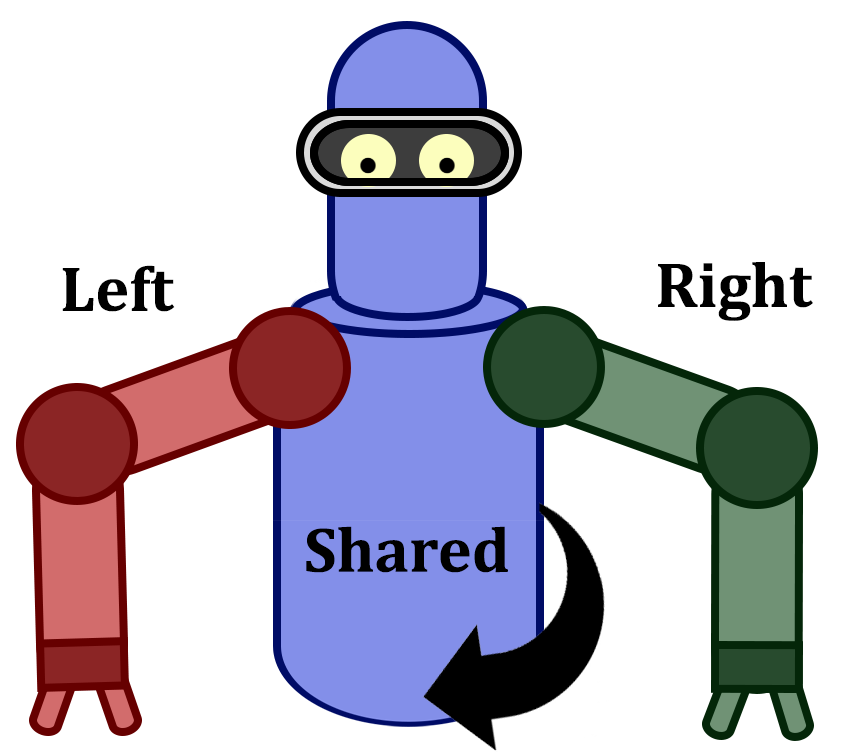}
		\caption{$\cfull = \cleft \times \cshared \times \cright$. }
	\label{fig:dof}
	\vspace{-0.2in}
\end{wrapfigure}

\begin{figure*}[!ht]
\centering
\includegraphics[height=2in]{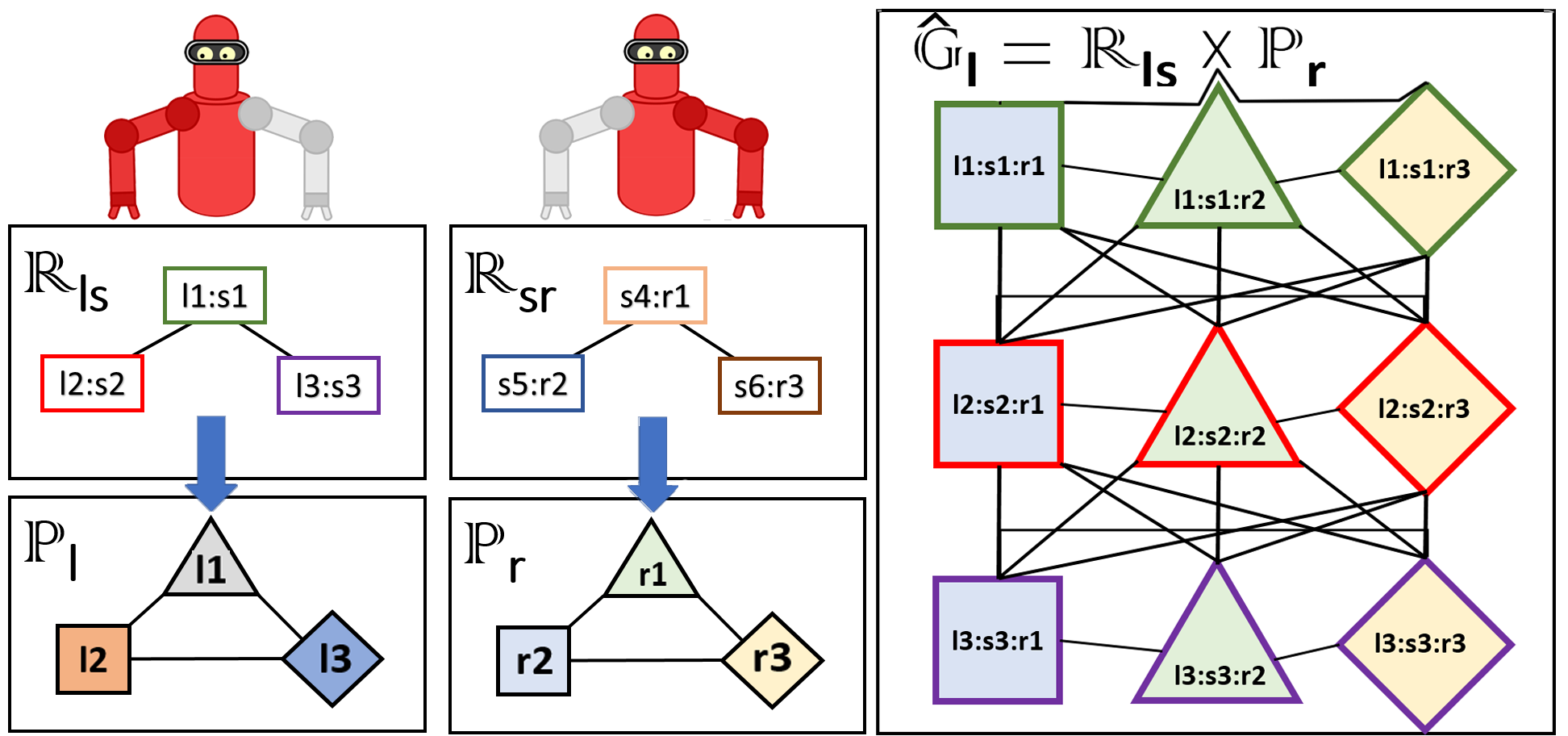}
\caption{The image on the left is an illustration of the decomposition
of the space to create arm-shared \dof\ roadmaps $ \mathbb{R} $ and
arm only roadmaps $ \mathbb{P} $. The example has three vertices in
each roadmap consisting of a combination Left($ l $), Shared($ s $),
and Right($ r $) values. For the sake of clarity the vertices on the arm-only roadmaps correspond to the $ \mathbb{R} $ roadmaps. The image on the right
shows the connectivity in the tensor product roadmap $ \gimpleft
= \lr \times \rp $. A similar tensor product is obtained
for \gimpright. }
\label{fig:dadrrt_decomposition}
\end{figure*}

The current work does not get into aspects related to
manipulation. Nevertheless, the primitives designed here can speed up
dual-arm manipulation task planning, where computational benefits can
be achieved by operating over multiple
roadmaps \citep{gravot2002playing, gravot2003method}. The topology of
dual-arm manipulation has been formalized \citep{koga1994multi,
Harada2014A-Manipulation} and extended to the $N$-arm
case \citep{Dobson:2015_MAM}. It requires the consideration of
multi-robot grasp planning \citep{vahrenkamp2010planning,
Dogar:2015aa}, regrasping \citep{vahrenkamp2009humanoid}, as well as
closed kinematic chain constraints \citep{Cortes:2004ff,
pallottinononinteracting}. Furthermore, force control strategies are
helpful for multi-arm manipulation of a common
object \citep{caccavale2008cooperative}. Recently coordinated control
has been applied to solve human-robot interaction \commentadd{tasks~\citep{mirrazavi2016coordinated}}. 

The algorithm is meant to address the applicability of $ \drrtstar $ to high dimensional humanoid robots with shared \dof.

\subsection{Problem Setup and Notation}

As shown in Fig. \ref{fig:dof}, the \dof $[F_1, \ldots, F_d]$ can be
grouped into \commentadd{left, right, and shared} \dof\ subsets, so that:
$\cfull = \cleft \times \cshared \times \cright$. 
A candidate solution path $\Sigma:[0,1]\rightarrow\cfree$ can be decomposed to
projections $[\Sigma_{l}, \Sigma_{s}, \Sigma_{r}]$ along
$\cleft$, $\cshared$ and $\cright$ respectively.

The method proposes the construction of the following roadmaps, as shown in
Fig. \ref{fig:dadrrt_decomposition}:

\begin{itemize}

\item A left-shared 
$ \mathbb{R}_{ls}(\nodes_{ls}, \edges_{ls}) $ and a right-shared \dof\
roadmap $ \mathbb{R}_{sr}(\nodes_{sr}, \edges_{sr}) $, where
$\nodes_{ls} \subset \cleft \times \cshared$ and
$\nodes_{sr} \subset \cshared \times \cright$. The edges are
collision-free paths in the same spaces, i.e., no collisions with the
static geometry, or self-collisions among the arm or the shared \dof
s.

\item A left arm 
$\lp(\nodes_{l}, \edges_{l})$ and a right arm roadmap
$\rp(\nodes_{r}, \edges_{r})$, such that $\nodes_{l} \subset \cleft$,
and $\nodes_{r} \subset \cright$. These roadmaps do not consider the
static geometry as they are not grounded by the shared \dof s. So, only
self-collisions between arm links are avoided.
\end{itemize}

The method focuses on two
\textit{tensor product roadmaps}: $ \gimpleft
= \mathbb{R}_{ls} \times \mathbb{P}_r $, and $ \gimpright
= \mathbb{R}_{sr} \times \mathbb{P}_l$. The method then simultaneously searches over $\gimpleft$ and $\gimpright$ in a $ \drrtstar $-esque fashion.

\subsection{Methodology}
This section describes the proposed method, and the way the
$ \dadrrtstar $ builds a forest of two trees \tree, which explores
both \gimpleft\ and \gimpright. In terms of the method's properties it
is sufficient to consider only one roadmap, but in practice, exploring
them simultaneously helps in the convergence, since we can evaluate
more possible solutions and rewires. The approach shows faster
convergence compared to $ \rrtstar $ in \cfull, and scales more than
$ \prmstar $.

At a high-level, the proposed Dual-arm \drrtstar\ (\dadrrtstar)
simultaneously explores the tensor product roadmaps \gimpleft\
and \gimpright, by building a search tree for each one so as to find a solution
from the start configuration $S$ to the target configuration $T$.  For
every vertex, the algorithm keeps track from which \textit{tensor product
roadmap} the vertex belongs to. Upon initialization, the tree starts
with two vertices, $S_l$ and $S_r$, one corresponding to tensor
product roadmap $ \gimpleft $ and the other to $ \gimpright $. Then, at every
iteration, the tree data structure \commentadd{$\tree$} is expanded by adding a
new edge and a node by calling an expand subroutine like Algorithm \ref{algo:drrtstar_expand}. The differences arises in the neighborhood calculation in Algorithm \ref{algo:drrtstar_expand} Line~8. The neighborhood $ N $ for $ \vnew $ considers the tensor roadmap neighborhoods that are part of the tree for both roadmaps. $ \vnew $ belongs to to either $ \gimpleft $ or $ \gimpright $. $ \hat{\vnew} $ is chosen to be the nearest tree vertex that was generated on the other tensor roadmap. $ N $ is the set of all tree vertices that are tensor roadmap neighbors of $ \vnew $ or $ \hat{\vnew} $. While doing rewires, care is taken to only rewire nodes belonging to the same tensor roadmap. The consideration of a richer neighborhood lets the algorithm ensure adequate exploration of both tensor roadmaps. The informed oracle $ \ioracle $ is similar to Algorithm \ref{algo:idrrtstar_oracle} with the difference arising for the constituent roadmap $ \mathbb{P} $, where the $ \heuristic $ estimate is simply the shortest Euclidean distance to the goal.

\noindent\textbf{Notes on Efficiency}: The difference of the
decomposition for shared degrees of freedom compared to $ \drrtstar $
is that $\mmgraph = \mathbb{R} \times \mathbb{P} $ does not give two
kinematically independent spaces. Specifically, $ \mathbb{P} $ depends
on the shared $ \dof $ to be grounded to the frame of the robot. This
means that the heuristic $ \heuristic $ is less informed for
$ \mathbb{P} $ and can only use the straight line
distance. \commentadd{The \drrtstar algorithm does not work out of the
box in the case of robots with shared degrees of freedom. The effect
of the less expressive heuristic in \dadrrtstar, translates into some
degradation in performance relative to the case of two kinematically
independent robotic arms. Nevertheless, \dadrrtstar\ is still
significantly faster than operating directly in the composite space of
the entire robot. There are not many methods that can practically
compute solutions for such high-dimensional (e.g., 15 degrees of
freedom) systems with kinematic dependences. The proposed \dadrrtstar\
method preserves some of the scalability benefits of \dadrrtstar and
addresses the kinematic dependence that arises for many popular
humanoid robots.}

\subsection{Analysis}
\noindent\textbf{Asymptotic optimality of tensor roadmaps : } 
Given the decomposition, $ \cfull $ is divided into
two parts: $\lr$ and $\rp$. 
\commentadd{If a robust optimal path $ \Sigma^* $ exists in $ \cfull $,} most of the arguments of Section \ref{sec:analysis} still hold for this decomposition.
Due to the nature of the space decomposition, since the constituent spaces do not correspond to kinematically independent robots, the clearance assumption in Claim \ref{claim:robust} needs to be reworked.

\begin{claimthm}
Robustness in $ \cfull $ implies robustness in $ \cspace_{ls} $ and $ \cspace_r $. For every decomposition, $\tau \in [0,1]$, and $q_i \in \cinv_i(\tau)$, 
$\| \sigma_i(\tau) - q_i \|_2 \geq \delta$.
\label{claim:robustdad}
\end{claimthm}
\begin{proof}

Consider any $Q=(\Sigma_{ls}(\tau), q_r)$, where $q_r$
is a configuration in $\cspace_r$ so that the right arm collides
either with the static geometry or with the left-shared part of the
robot, which is at $\sigma_{ls}(\tau)$. Given a robust $\Sigma$, $Q$ is a colliding configuration: $\delta \leq \|\Sigma(\tau)-Q\|$. But $Q$ and
$\Sigma(\tau)$ only differ in $q_r$, so the path $\sigma_r$ has
clearance $\delta$ 
$$\delta \leq
|| \sigma_r(\tau) - q_r||.$$

By switching the decomposition of $ \Sigma $ in $ \gimpleft $ into $
(\sigma_{l}, \sigma_{sr}) $, by the above reasoning:
  
 \begin{align*}
 \delta \leq || \Sigma(\tau) - Q || 
 \implies \delta &\leq || \Sigma_l(\tau) - q_l||\\
 \text{Now, since~ }\forall \tau:\ \ \ 
 || \sigma_{ls}(\tau) - q_{ls}|| &\geq || \sigma_l(\tau) - q_l|| \\ 
 \implies \delta &\leq || \sigma_{ls}(\tau) - q_{ls}||.
 \end{align*}
 
This proves the robustness for $ \sigma_{ls} $. The same reasoning can be applied to $ \cspace_{sr} $ and $ \cspace_l $.

\end{proof}

It suffices to follow the proof structures outlined in Section \ref{sec:analysis} to argue asymptotic optimality for the method. 
It should be noted that due to the coupled nature of $ \cfull $ introduced by the shared $ \dof $, the use of the Euclidean cost metric is more applicable.

\section{Experimental Validation}
\label{sec:experiments}

This section provides an experimental evaluation of \drrtstar\ by
demonstrating practical convergence, scalability for disk robots, and
applicability to dual-arm manipulation.  \commentadd{The choice of a
  cost metric depends on the type of application and the underlying
  system properties. For systems without shared degrees of freedom,
  the considered cost function is the sum of individual Euclidean arc
  lengths, which is a popular choice for multi-robot systems. For
  systems with shared degrees of freedom, the combined nature of the
  underlying configuration space motivates the use of Euclidean arc
  length in the composite space as the metric.  The results show that
  the properties and benefits of the proposed algorithms stay robust
  for both choices of cost functions.}

\subsection{Tests on Systems without Shared \dof}

The approach and alternatives are executed on a cluster with Intel(R)
Xeon(R) CPU E5-4650 @ 2.70GHz processors, and 128GB of RAM. The
solution costs are evaluated in terms of the sum of Euclidean arc
lengths.

\noindent \textbf{2 Disk Robots among 2D Polygons:} This
base-case test involves $ 2 $ disks ($\cspace_i := \reals^2$) of
radius $0.2$ with bounded velocity, in a $10 \times 10$ region,
inflated by the radius, as in Figure~\ref{fig:poly_enviro}. The disks
have to swap positions between $(0,0)$ and $(9,9)$. This is a setup
where it is possible to compute the explicit roadmap, which is not
practical in more involved scenarios. In particular, \drrtstar\ is
tested against: a) running $\astar$ on the implicit tensor roadmap
$\mmgimp$ (referred to as ``Implicit \astar''), where $\mmgimp$ is
defined over the same individual roadmaps with $N$ nodes as those used
by \drrtstar; b) an explicitly constructed \prmstar\ roadmap with
$N^2$ nodes in $\cspace$; and c) the $ \udrrtstar $ variant of the algorithm.

Results are shown in Figure~\ref{fig:polygonal_benchmark}.  \drrtstar\
converges to the optimal path over $\mmgimp$, similar to the one
discovered by Implicit \astar, while quickly finding an initial
solution of high quality. Furthermore, the implicit tensor product
roadmap $\mmgimp$ is of comparable quality to the explicitly
constructed roadmap. 
The convergence of $ \drrtstar $ is faster compared to corresponding $ \udrrtstar $ variant as evident from Figure~\ref{fig:polygonal_benchmark}(left).

\begin{figure}[!t]
	\centering
	\includegraphics[width=1.6in]{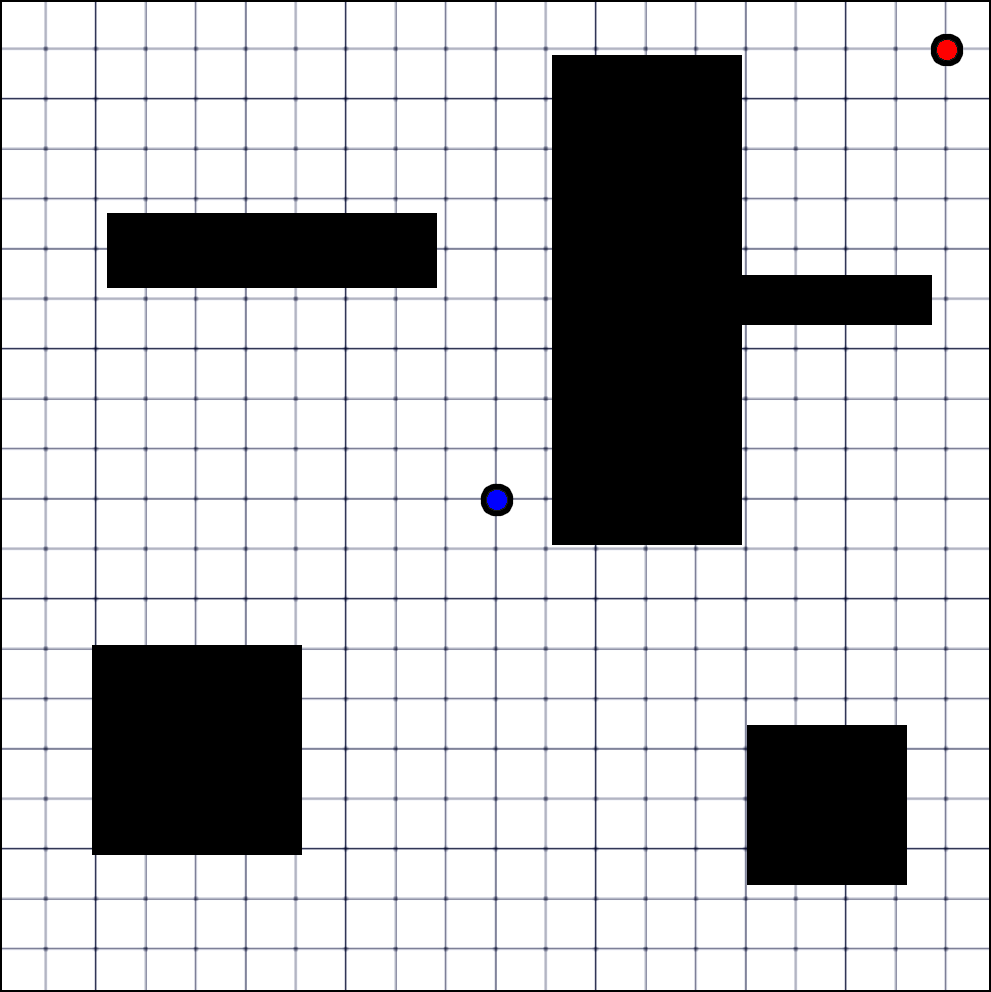}
	\caption{The 2D environment where the 2 disk robots 
		operate.}
	\label{fig:poly_enviro}
\end{figure}

Table~\ref{tab:2_robot} presents running times.
\drrtstar\ and implicit \astar\ construct $2$ $N$-sized roadmaps 
(row~3), which are faster to construct than the \prmstar\ roadmap in
$\cspace$ (row~1).  \prmstar\ becomes very costly as $N$ increases.
For $N=500$, the explicit roadmap contains $250,000$ vertices, taking
$1.7$GB of RAM to store, which was the upper limit for the machine
used. When the roadmap can be constructed, it is fast to query
(row~2). \drrtstar\ quickly returns an initial solution (row~5), at par with the solution times from the explicit roadmap and well before
Implicit \astar\ returns a solution (row~4). The initial solution times are compared visually in Figure~\ref{fig:polygonal_benchmark} which demonstrates the efficiency of $ \drrtstar $ compared to $ \udrrtstar $ as well. The next
benchmark further emphasizes this point.

\commentadd{The comparison between the early solution time required to
  find a suboptimal solution by the proposed method against the
  computation time needed by the optimal \astar\ highlights the impact
  of roadmaps of increasing sizes. While \drrtstar 's initial solution
  times barely change, the time taken by any variant of heuristic
  search over the composite roadmap increases with the size of the
  roadmap. This indicates that roadmaps of size similar to the tensor
  roadmaps considered here would rapidly cease to be solvable without
  anytime performance similar to that of \drrtstar.}

\begin{table}[!h]
\centering
\caption{Construction and query \commentadd{times (seconds)} for 2 disk robots.}
\label{tab:2_robot}
\small
\begin{tabular}{|l|c|c|c|}
\hline
\multicolumn{1}{|r|}{\textbf{Number of nodes: $ N $ =}} & \textbf{50} & \textbf{100} & \textbf{200} \\ \hline
{$N^2$-PRM* construction}                     & 3.427        & 13.293        & 69.551        \\ \hline
{$N^2$-PRM* query}                            & 0.002       & 0.005        & 0.019        \\ \hline
{2 $N$-size PRM* construction}              & 0.135        & 0.274        & 0.558       \\ \hline
{Implicit A* search over $\mmgimp$}                     & 0.886       & 4.214        & 15.468        \\ \hline
{\udrrtstar\ over $\mmgimp$ (initial) }                    & 1.309       & 0.999        & 0.638        \\ \hline
{\drrtstar\ over $\mmgimp$ (initial) }                    & 0.003       & 0.002        & 0.002        \\ \hline
\end{tabular}
\end{table}

\begin{figure*}[h]
	\centering
		\includegraphics[height=2in]{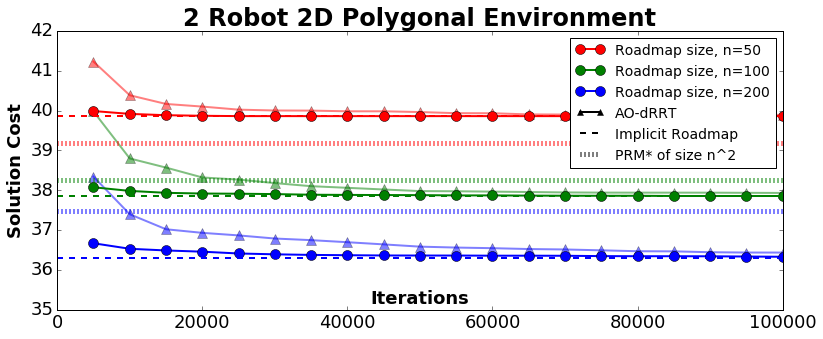}
		\includegraphics[width=4.6in,height=2in]{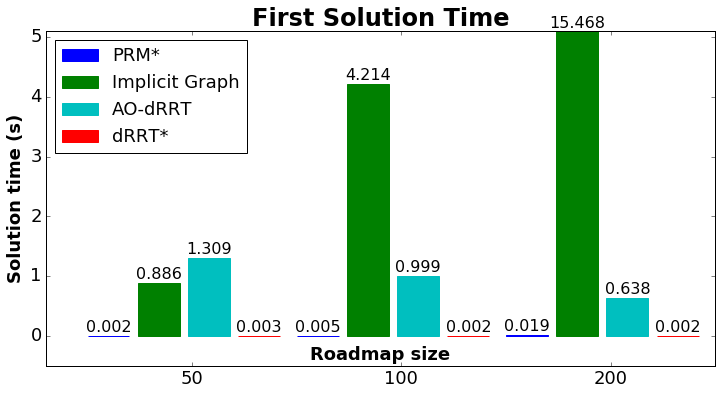}
		\caption{\commentadd{For every $ n  $, $ 10 $ randomly generated pairs of roadmaps 
		are generated. $ \drrtstar $ and $ \udrrtstar $ ran on 5 random experiments for every roadmap pair,
		and the implicit $ \astar $ searches these $ 10 $ tensor combinations. $ \prmstar $
		is run $ 10 $ times for every $ n $.  
		\commentadd{\textit{(Top)}: Average} solution cost is reported over iterations. Data averaged over $10$
		roadmap pairs.  $\drrtstar$ (solid circled line) and $ \udrrtstar $ (solid triangled line) converges to the optimal path
		through $\mmgimp$ (dashed line). \textit{(Bottom)}: Initial solution times for the algorithms.} }
	\label{fig:polygonal_benchmark}
\end{figure*}

\begin{figure*}[!ht]
\centering
\includegraphics[width=0.9\textwidth]{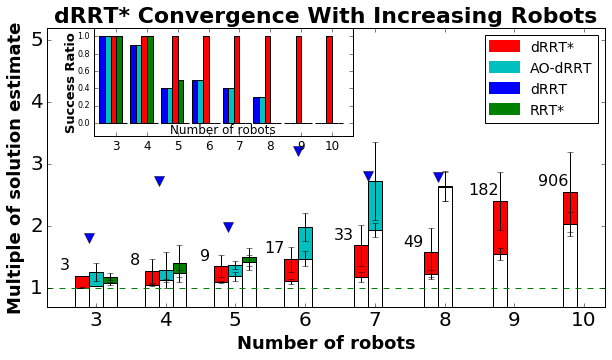}
\includegraphics[width=0.45\textwidth]{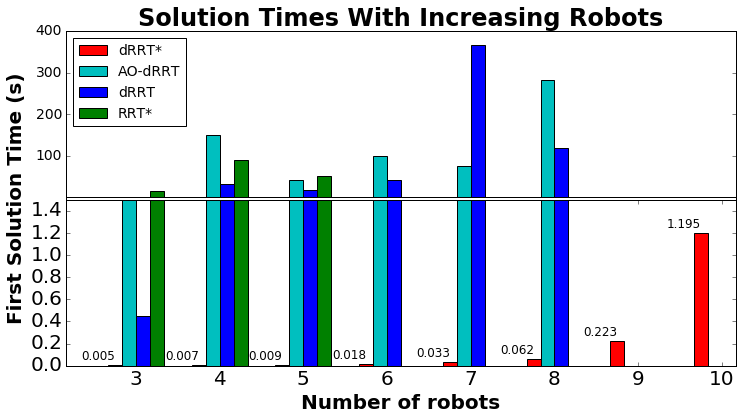}
\includegraphics[width=0.45\textwidth]{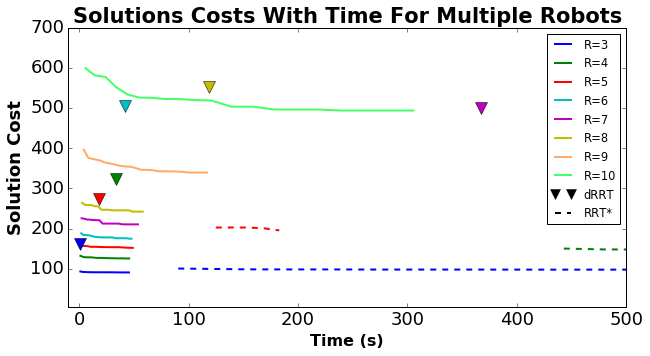}

\caption{\commentadd{Data averaged over $10$ runs for $R = 3$ to $10 $
    robots. The data is reported for the algorithms \drrtstar,
    \udrrtstar, \drrt, and \rrtstar.  (\textit{Top inset}): The
    success ratio shows the fraction of the runs that returned a
    solution.  (\textit{Top}): Relative solution cost of the
    algorithms for increasing $R$ over $100,000$ iterations. The
    horizontal green line at $1$ denotes the best possible cost
    estimate, which is a combination of the individual robot shortest
    paths for each problem. All the other costs are represented as
    multiples of this estimate.  \drrt\ only reports a single
    solution, and is denoted by the inverted blue triangles. The other
    algorithms improve the solution over the iterations, represented
    by vertical bars between the average initial solution cost and
    average final reported solution cost. The numbers above the
    \drrtstar\ bars represent the iteration number of the first
    solution for \drrtstar.  (\textit{Bottom left}): The average
    initial solution times for the algorithms. (\textit{Bottom
      right}): The plot of the reported solution cost over time for
    the different algorithms. \drrt\ only reports a single solution
    and is represented by the inverted triangles.  }}
\label{fig:scalability}
\end{figure*}

\textbf{Many Disk Robots among 2D Polygons:} 
In the same environment as
above, the number of robots $R$ is increased to evaluate scalability. 
\commentadd{The same environment is maintained in this benchmark to introduce additional complexity purely in terms of the addition of more robots into the planning problem. The effect of more difficult and practical planning scenarios would be explored in the subsequent benchmarks with manipulators.}
Each robot starts on the perimeter of the environment and is tasked
with reaching the opposite side. An $N=50$ roadmap is constructed for
every robot. It quickly becomes intractable to construct a \prmstar\
roadmap in the composite space of many robots.

Figure~\ref{fig:scalability} shows the inability of alternatives to
compete with \drrtstar\ in scalability. Solution costs are normalized
by an optimistic estimate of the path cost for each case, which is the
sum of the optimal solutions for each robot, disregarding robot-robot
interactions.  The colored vertical bars represent the range of the average initial and final solution costs. Implicit \astar\ fails to return solutions even for 3
robots. Directly executing \rrtstar\ in the composite space fails to
do so for $R \geq 6$.  The original \drrt\ method (without the
informed search component) starts suffering in success ratio for
$R \geq 4$ and returns worse quality solutions than \drrtstar. The $ \udrrtstar $ variant performs similar to $ \drrt $ in terms of success ratio but expectedly finds better solutions than $ \drrt $. $ \drrtstar $ finds solutions up to $ R=10 $. 

\commentadd{In order to give an estimate of the immensity of the size
  of the search space,} for $ R=10 $, the \textit{tensor-product
  roadmap} represents an implicit structure consisting of $ 50^{10} $
or $ \mathtt{\sim}100 $ million-billion vertices.

\textbf{Dual-arm manipulator:} This test (Figure~\ref{fig:motoman_start_goal}) shows the benefits
of \drrtstar\ when planning for two $7$-dimensional arms.
Figure~\ref{fig:motoman_convergence} shows that
\rrtstar\ fails to return solutions within $100K$
iterations. Using small roadmaps is also insufficient for this
problem.  Both \drrtstar\ and Implicit \astar\ require larger roadmaps
to begin succeeding. But with $N \geq 500$, Implicit \astar\ always
fails, while \drrtstar\ maintains a $100\%$ success ratio. As
expected, roadmaps of increasing size result in higher quality
path. The informed nature of \drrtstar\ also allows to find initial
solutions fast, which together with the branch-and-bound primitive
allows for good convergence. The initial solution times in Figure~\ref{fig:motoman_convergence} indicate that the heuristic guidance succeeds in finding fast initial solutions even for larger roadmaps.

\begin{figure}[h]
 \includegraphics[width=0.23\textwidth]{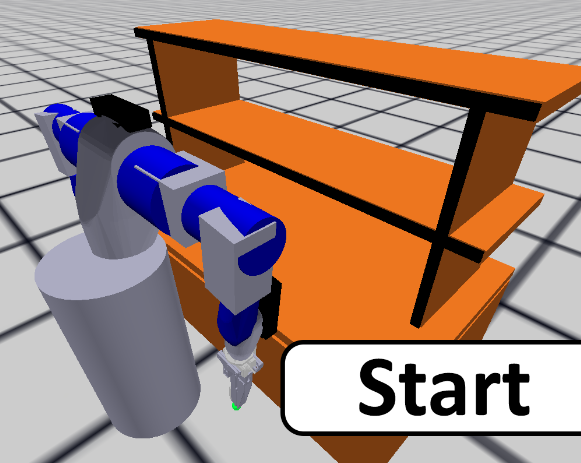}
 \includegraphics[width=0.23\textwidth]{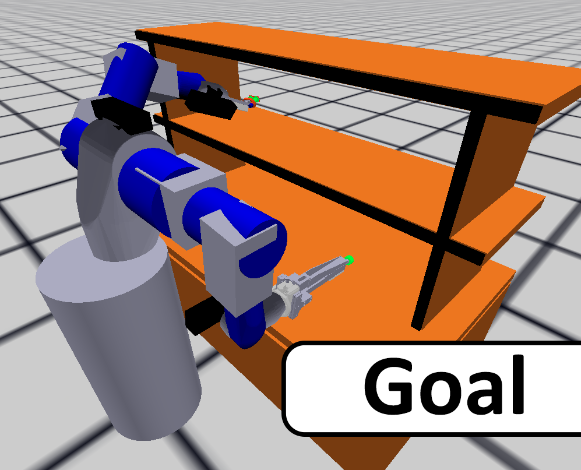}
 \caption{The start and target configuration for the dual-arm manipulator benchmark on a $ \motoman $ {\tt SDA10F}. \commentadd{\drrtstar\ is run for a dual-arm manipulator to go from its home position (left) to a reaching 
 	configuration (right).}
 }
 \label{fig:motoman_start_goal}
\end{figure}
\begin{figure}[ht]
\centering

\includegraphics[width=0.45\textwidth]{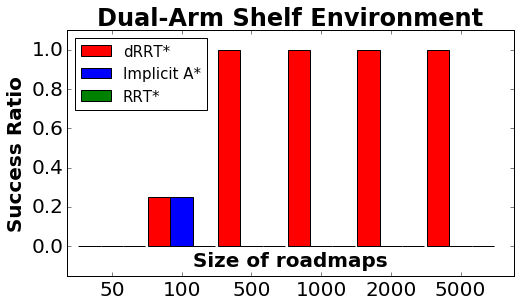}
\vspace{0.1in}

\includegraphics[width=0.45\textwidth]{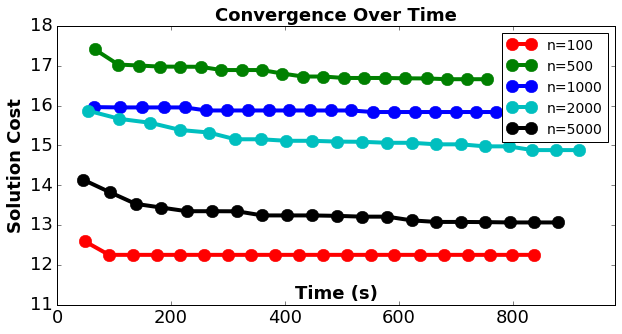}

\vspace{0.1in}
\includegraphics[width=0.45\textwidth]{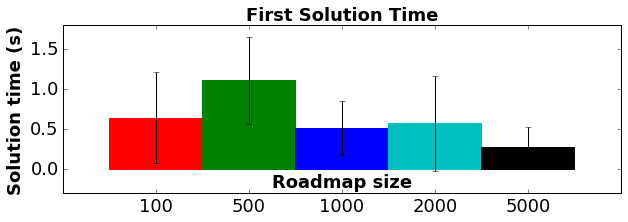}
\caption{
(\textit{Top}): $ 5 $ random experiments are run for $ 4 $ random
roadmap pairs for every $ n $. $ \drrtstar $ achieves 
perfect success ratio as $n$ increases.
(\textit{Middle}): \drrtstar\ solution quality over time.  Here,
larger roadmaps provide benefits in terms of running time and solution
quality.
(\textit{Bottom}): Initial solution times for $ \drrtstar $.}
\label{fig:motoman_convergence}
\end{figure}

\subsection{Tests on Systems with Shared \dof}

This section showcases three benchmarks of increasing difficulty,
which are used to evaluate the performance of
the \dadrrtstar. All the experiments were run on a cluster with
Intel(R) Xeon(R) CPU E5-4650 @ 2.70GHz processors, and 128GB of
RAM. In each benchmark, different sizes $ n $ of the constituent
roadmaps $ \lr $ and $ \rr $ were evaluated.  The \dadrrtstar
algorithm is compared against \rrtstar and $ \prmstar $.  The platforms used are $ \motoman $ {\tt SDA10F}, with a torsional \dof, and $ \baxter $ on a mobile base that can rotate and
translate. For the $ \prmstar $ algorithm and all benchmarks, $ 20 $ randomly seeded
roadmaps with $ 50,000 $ nodes are constructed in $ \cfull $ and data
are gathered from $ 20 $ experiments.  A $ 50,000 $ node roadmap has
$ \approx1 $ million edges, and takes $ \approx 7$ hours to construct
in these high dimensional spaces. Larger roadmaps run into memory
scalability issues. These roadmaps in the full space occupied
$ \approx50 $MB. In comparison, the space requirement for two arm
roadmaps were $ <1 $MB.

For all benchmarks, both $ \rrtstar $ and \dadrrtstar were allowed to run
for $ 100,000 $ iterations. $ \rrtstar $ is ran in $ 20 $ different
randomly seeded experiments for every benchmark. For the \dadrrtstar\ algorithm, $ 20 $ experiments are run for every
benchmark, for the different constituent roadmap sizes $ n $, by
building $ 4 $ pairs of randomly seeded constituent roadmaps, and
running $ 5 $ randomly seeded experiments over each roadmap
combination. 

\textbf{Motoman Tabletop Benchmark:} A set of $ 20 $ random
collision-free starts and goals are selected in the tabletop
environment, shown in Fig. \ref{fig:benchmarks}.

They are only used if
they are sufficiently far away from each other. \dadrrtstar is tested
with constituent roadmap sizes of $ 100, 250$ and $ 500 $. All the
algorithms succeed in every experiment. In this simpler problem,
smaller roadmaps are quicker to search, and generate initial solutions
faster compared to $ \rrtstar $, as shown in
Fig. \ref{fig:motoman_tabletop} (\textit{top}). 
\begin{wrapfigure}{r}{0.235\textwidth} 	\centering
	\includegraphics[width=0.235\textwidth]{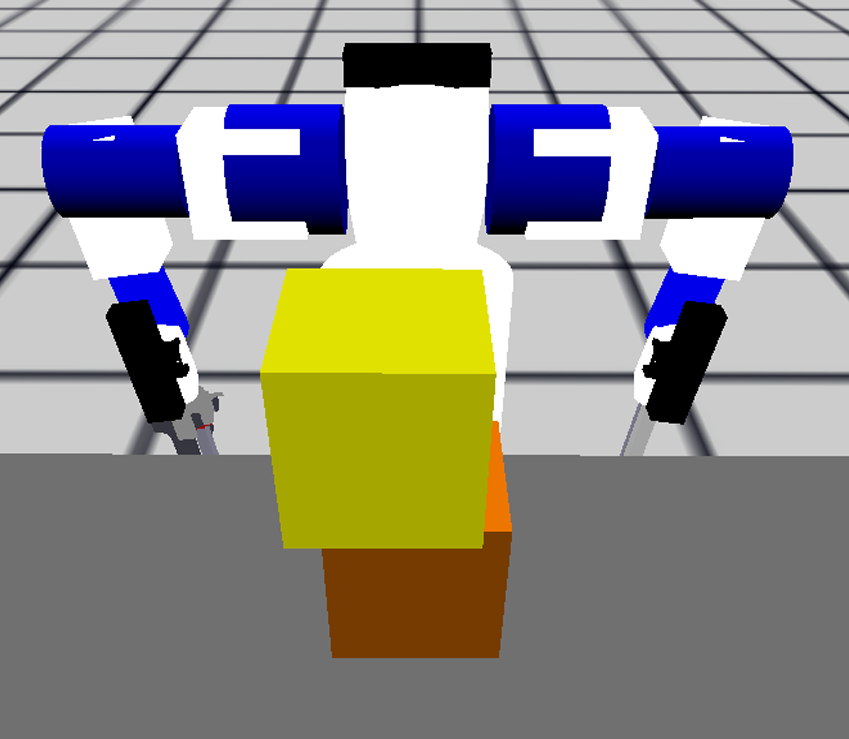}
		\caption{Motoman tabletop.}
	\label{fig:benchmarks}
	\end{wrapfigure}
Searching the
$ \prmstar $ is the fastest (online), but the solution quality is
worse than that obtained from the other methods. \dadrrtstar\ converges
to better solutions, compared to the other algorithms, as shown in
Fig \ref{fig:motoman_tabletop} (\textit{bottom}).

\begin{figure}[!h]
	\centering
	\includegraphics[width=3in]{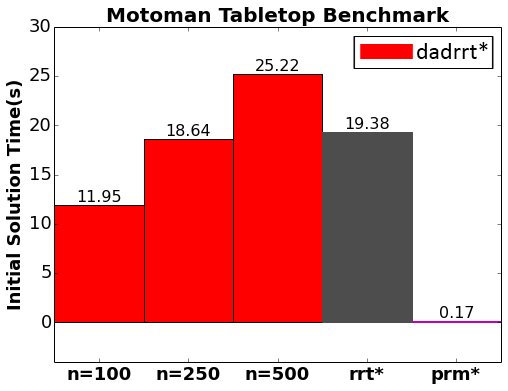}
	\includegraphics[width=3in]{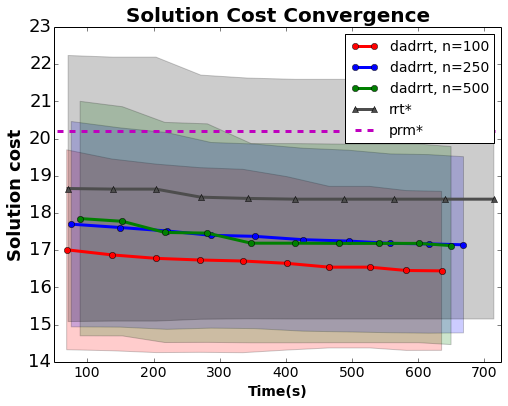}
	\caption{Motoman Tabletop Benchmark: \textit{Top}: The initial solution times are reported for every algorithm. \textit{Bottom}: The average solution costs over time are reported. 
			}
	\label{fig:motoman_tabletop}
\end{figure}

\commentadd{\textbf{Motoman~Shelf~Benchmark}: This} benchmark sets up the
$ \motoman $ in front of $ 3 $ shelves. The robot has to plan between
two states where both arms are inside different shelving units, which
require the rotation of its torso \commentadd{(Fig. \ref{fig:motoman_shelf} (\textit{top}))}.

This is a significantly harder problem, and $ \rrtstar $ suffers in terms
of success ratio \commentadd{(Fig. \ref{fig:motoman_shelf} (\textit{second}))}. $ \rrtstar $ takes much longer to find the initial solution, as
indicated by Fig. \ref{fig:motoman_shelf} (\textit{middle}). $ \prmstar $ is still the fastest in finding solutions (only
online cost considered again here). The \dadrrtstar solution cost is
much better than both the average $ \prmstar $ solution, and \rrtstar,
as shown in Fig. \ref{fig:motoman_shelf} (\textit{bottom}). \dadrrtstar\ will quickly converge for smaller
roadmaps, and then stop improving the cost. The larger roadmaps
contain better solutions, causing \dadrrtstar to converge slower. 

\begin{figure}[!h]
\centering
\includegraphics[width=2.7in]{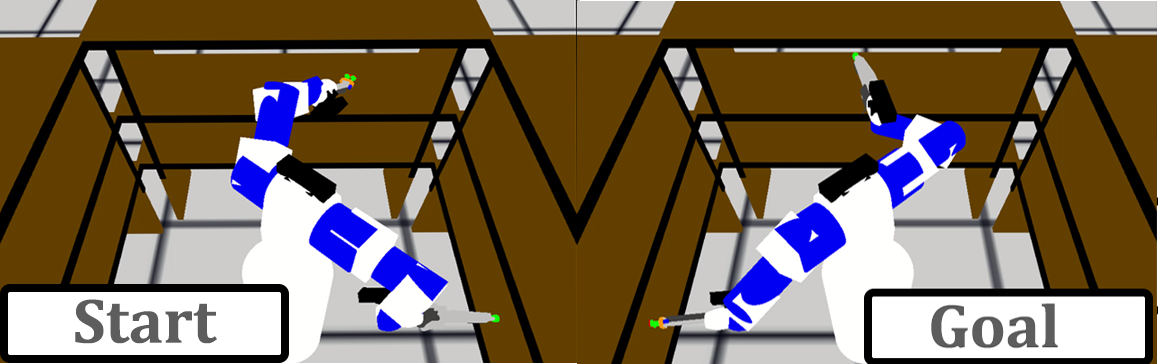}
\includegraphics[width=2.9in]{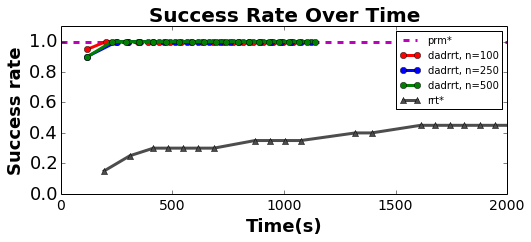}
\includegraphics[width=2.7in]{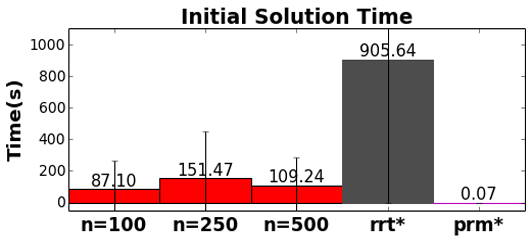}
\includegraphics[width=2.9in]{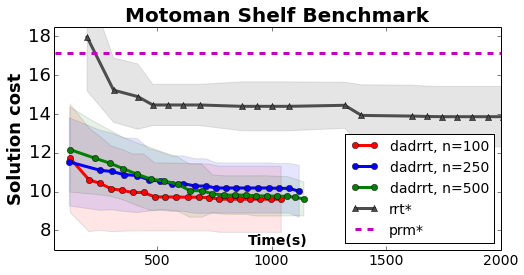}

\caption{Motoman Shelf Benchmark: \textit{Top:} The setup of the benchmark. 
\textit{Second:} Success ratios of
the algorithms are shown over time. \textit{Middle: } The initial
solution times are reported for every algorithm. \textit{Bottom:} The
average solution costs over time are reported. The dashed horizontal
line denotes the average solution cost discovered by $ \prmstar $. The
shaded regions represent the corresponding algorithm's standard
deviation of cost. }
\label{fig:motoman_shelf}
\end{figure}

\commentadd{\textbf{Mobile Baxter Benchmark}: This} benchmark uses
a \textit{Rethink} $ \baxter $ robot with a mobile base. The robot is
grasping two long objects inside a shelf Fig. \ref{fig:mobile_baxter} (\textit{top}). The robot has to
navigate across a cramped, walled room, to a placing configuration
inside a shelf on the other side of the room.

This proves to be the most challenging problem among the three
benchmarks. As shown in Fig. \ref{fig:mobile_baxter} (\textit{middle}), \rrtstar fails to find a solution. It should be noted that,
when tested on a simpler version of the benchmark without the pillar
in the room, \rrtstar could find solutions. $ \prmstar $ also falters
by showing a very low success rate. This indicates that we need even
larger roadmaps in $ \cfull $ to solve harder problems. The
problem is solved when a dense implicit structure, with $ n=1000 $ is explored
by \dadrrtstar. 

Fig.~\ref{fig:mobile_baxter} (\textit{bottom}) shows that \dadrrtstar\
finds better initial and converged solutions when compared to the
instances in which $ \prmstar $ succeeded.

\begin{figure}[!h]
\centering
\includegraphics[height=1.1in]{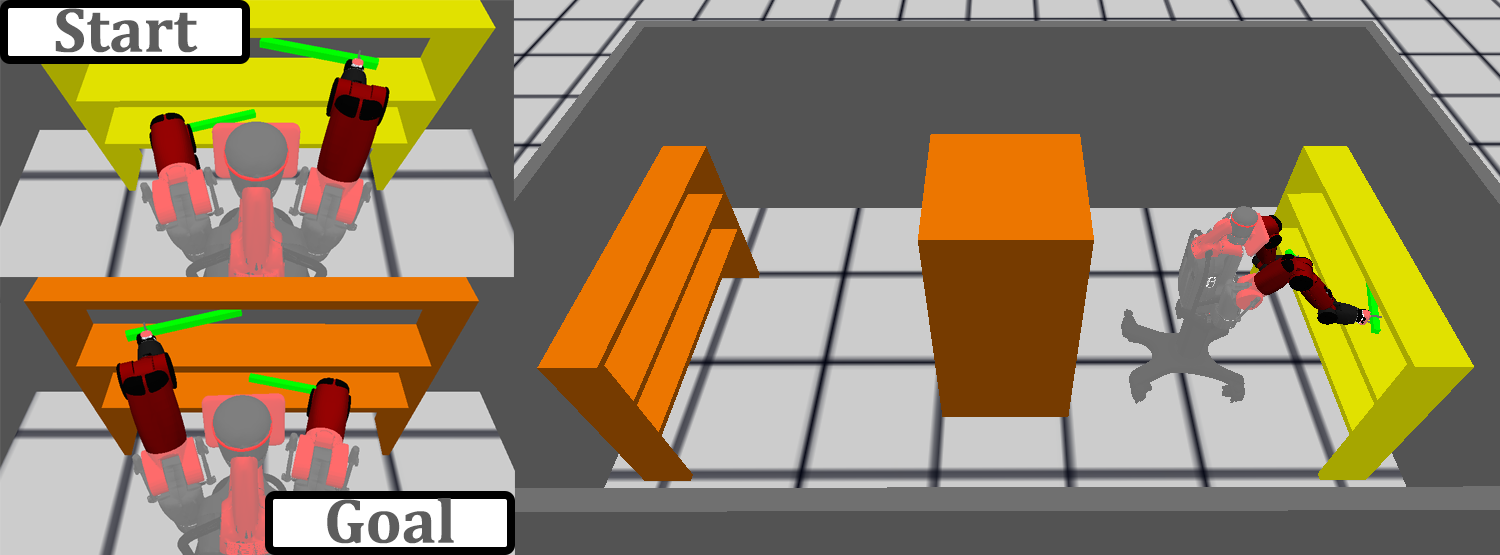}
\includegraphics[width=3in]{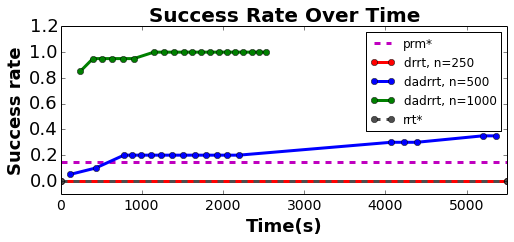}
\includegraphics[width=3in]{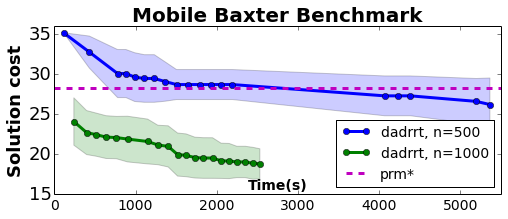}
\caption{Mobile Baxter Benchmark: \textit{Top:} The setup of the benchmark. \textit{Middle:} Success ratios of the algorithms are shown over time. \textit{Bottom:} The average solution costs over time are reported. The dashed horizontal line denotes the average solution cost discovered by $ \prmstar $. 
}
\label{fig:mobile_baxter}
\end{figure}

\subsection{Real world experiments}
Experiments were performed \commentadd{in a 28}-dimensional space with
two dual-armed manipulators: (a $ \motoman $ {\tt{SDA10f}} and a $
\baxter $). \commentadd{Initial solutions were obtained in a fraction
  of a second for the two experimental setups, with the method allowed
  to run for $1000$ iterations to improve the quality of the
  demonstrated trajectories.  The two setups are chosen carefully to
  demonstrate in the first instance a typical application of
  simultaneous grasping that may arise in real world scenarios, and in
  the second instance a problem that forces very close interactions
  between the arms in close proximity.}

\noindent\textbf{Pre-grasp Demonstration}: As shown in Figure \ref{fig:realworld1}, the demonstration simulates an application to multi-arm manipulation, where the goals of the motion planning problem for 4 arms is to pre-grasping configurations for 4 objects placed on a table in the shared workspace between the robots. $ 1000 $ node roadmaps were constructed for each arm and $ \drrtstar $ was used to search for a solution to the motion planning problem. The solution was computed offline and an open-loop execution was performed on the real system.

\noindent\textbf{Coupled Workspace Demonstration}: As shown in \commentadd{Figure \ref{fig:realworld2}, a pole} is positioned between the two robots so that the arms cannot cross over. The objective is for the 4 arms to a) approach the pole at alternating heights, b) then swap the height of their approaching configurations, and c) finally return back to the start state. $ 1000 $ node roadmaps were constructed for each arm and $ \drrtstar $ was used to search for a solution to the three motion planning problems. The solutions that were computed offline, were stitched together and replayed in an open loop execution on the real system.

\begin{figure*}[ht]
\centering
\includegraphics[width=0.98\textwidth]{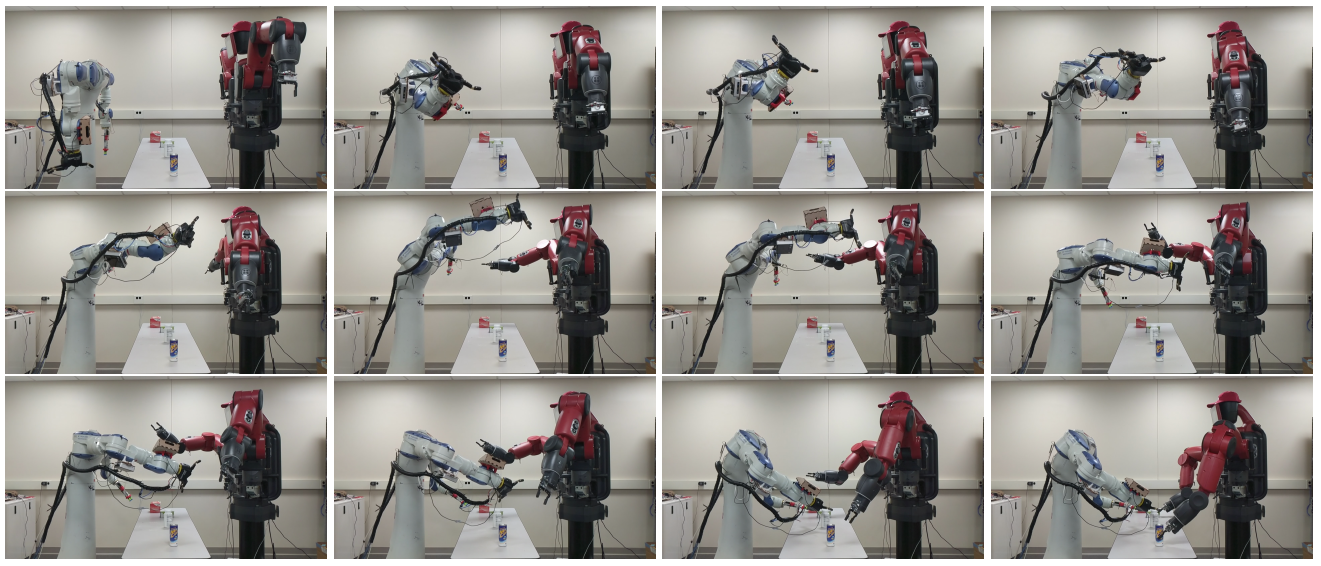}
\caption{Real world experiments were performed \commentadd{in a 28}
  dimensional space with 2 dual-armed manipulators planning their
  motion to a goal configuration corresponding to pre-grasping states
  for 4 objects resting on a table in the shared workspace.
  \commentadd{The sequence corresponds to freeze-frames starting in
    sequence from the top-left, and progressing along each row till
    the bottom-right.}}
\label{fig:realworld1}
\end{figure*}

\begin{figure*}[ht]
\centering
\includegraphics[width=0.98\textwidth]{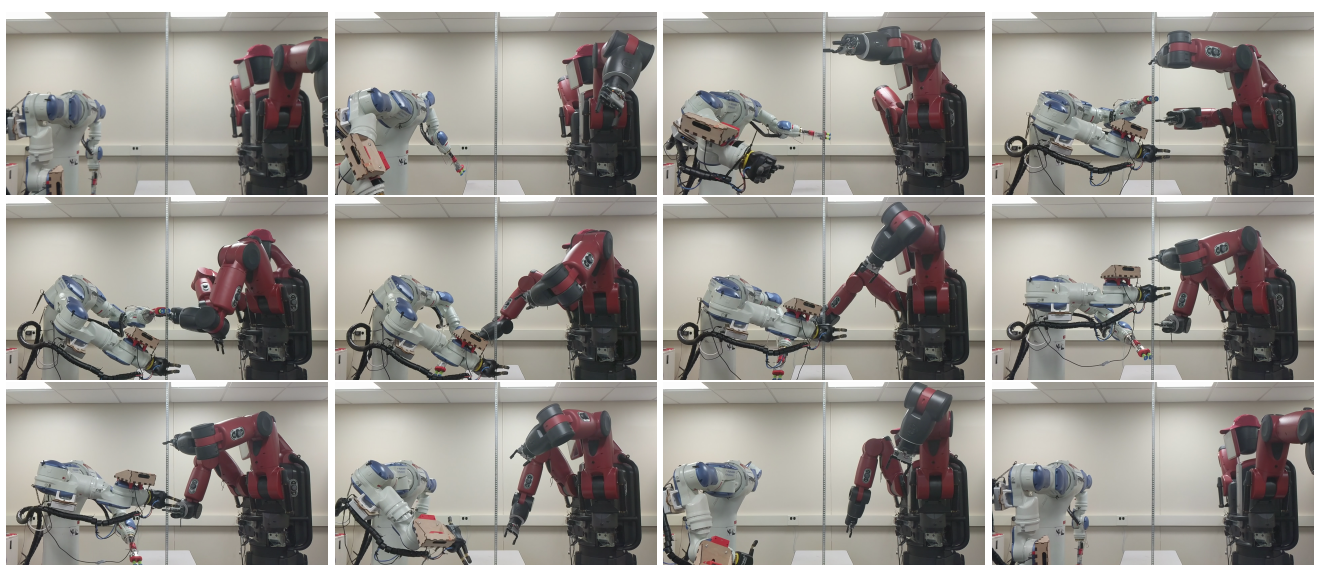}
\caption{Real world experiments were performed \commentadd{in a 28}
  dimensional space with 2 dual-armed manipulators planning their
  motion from a start state to an approach state close to a pole
  positioned in the center of a tightly coupled shared workspace. The
  arms then swap positions on the pole and return back to the start
  state. The goals correspond to the last images in each row.
  \commentadd{The sequence corresponds to freeze-frames starting in
    sequence from the top-left, and progressing along each row till
    the bottom-right.} }
\label{fig:realworld2}
\end{figure*}
 
\section{Discussion}
\label{sec:discuss}

This work proves asymptotic optimality of sampling-based multi-robot
planning over implicit structures by extending the \drrt\ approach.
Asymptotic optimality is achieved by making a modification, resulting
in \udrrtstar\ which expands a spanning tree over an implicitly
defined tensor product roadmap, and leverages a simple re-wiring
scheme.

This method already has the advantage of avoiding the construction of
a large, dense roadmap in the composite configuration space of many
robots. This can be further improved to use heuristics so as to
search in an informed manner, in the \drrtstar\ method. The
method is also extended to work with robot systems, which share
degrees of freedom, resulting in \dadrrtstar.

\commentadd{Experimental results show the efficacy of the proposed approaches.}
Furthermore, by leveraging
heuristics, \drrtstar\ is able to solve more challenging
problem instances than the baseline \udrrtstar\ method, and the
approach is demonstrated to solve complex, real-world problems with robot manipulators operating in a shared workspace with
a high degree of coupling.

\commentadd{In terms of practical applicability, \drrtstar\ promises
  fast initial solutions times (Figures \ref{fig:scalability} and
  \ref{fig:motoman_convergence}) on the order of a fraction of a
  second for most problems, including for high-dimensional,
  kinematically independent multi-robot problems, which is an exciting
  result. The solution quality improvement indicates the anytime
  properties of the approach, where paths of improved path quality are
  discovered as more computation time is invested. While problems with
  shared degrees of freedom provide less guidance and result in
  performance degradation, the scalability benefits remain even in
  this case relative to composite planning.} \commentadd{Future work
  includes the consideration of dynamics. The existing theoretical
  analysis of \drrtstar\ assumes that the individual robot systems are
  holonomic, which guarantees the existence of near-optimal
  single-robot paths (see Lemma~\ref{lem:prm}
  and~\cite[Theorem~4.1]{Pavone:2015fmt}). Recent results concerning
  asymptotic optimality of \prm\ for non-holonomic systems
  (see~\cite{ELM15a,ELM15b}) bring the hope of achieving a more
  general analysis of the current work as well.}  The proposed
framework can also be leveraged toward efficiently solving
simultaneous task and motion planning for many robot manipulators
\citep{Dobson:2015_MAM}. The demonstrated applications to manipulators
also motivate dual-arm rearrangement challenges~\citep{shome2018fast}.
 
{
\small
 \bibliographystyle{spbasic}

}

\end{document}